\documentclass[11pt,letterpaper]{article}

\usepackage[margin=1in]{geometry}        
\usepackage{setspace}                     
\usepackage{booktabs}                      
\usepackage[T1]{fontenc}                  
\usepackage[utf8]{inputenc}               
\usepackage{lmodern} 
\usepackage{graphicx}
\usepackage{natbib}
\usepackage{authblk}
\usepackage{enumerate}
\usepackage{enumitem}

\usepackage{amsmath,amssymb,amsthm,mathtools,bbm}

\numberwithin{equation}{section}

\theoremstyle{plain}
\newtheorem{theorem}{Theorem}
\newtheorem{lemma}[theorem]{Lemma}
\newtheorem{proposition}[theorem]{Proposition}

\theoremstyle{definition}
\newtheorem{definition}{Definition}
\newtheorem{assumption}{Assumption}

\usepackage[ruled, linesnumbered]{algorithm2e}
\usepackage{xcolor}

\usepackage[
  colorlinks,
  linkcolor=blue,
  citecolor=blue,
  urlcolor=blue,
  breaklinks
]{hyperref}

\def\E{\mathbb{E}}
\def\P{\mathbb{P}}
\def\R{\mathbb{R}}
\def\N{\mathbb{N}}
\def\Xcal{{\mathcal{X}}}
\newcommand{\bx}{\mathbf{x}}
\newcommand{\bX}{\mathbf{X}}
\newcommand{\bY}{\mathbf{Y}}
\newcommand{\bt}{\mathbf{t}}
\newcommand{\bxi}{\boldsymbol{\xi}}

\newcommand{\bv}{\mathbf{v}}
\def\Hcal{\mathcal{H}}
\def\Ccal{\mathcal{C}}
\def\Pcal{\mathcal{P}}

\newcommand{\Pit}{\Pi_{\bt}}
\newcommand{\One}[1]{{\mathbbm{1}}\left\{{#1}\right\}}
\newcommand{\one}[1]{{\mathbbm{1}}_{{#1}}}
\def\conch{{\textsc{CONCH}}}
\def\clocc{{\textsc{CLOCC}}}

\newcommand{\dfucb}{\textsc{DF-UCB}\,}
\newcommand{\dflcb}{\textsc{DF-LCB}\,}

\usepackage[nameinlink,noabbrev]{cleveref}
\crefname{assumption}{assumption}{assumptions}

\newcommand{\papertitle}{Distribution-free inference on the number of changepoints}
\newcommand{\paperauthorA}{Rohan Hore}
\newcommand{\paperauthorB}{Aaditya Ramdas}
\newcommand{\affilOne}{Department of Statistics, Stanford University}

\newcommand{\corrEmail}{rhore@stanford.edu}
\newcommand{\keywordslist}{Changepoint count, number of changepoints, conformal inference, conformal p-values,\\ distribution-free inference.}

\title{\papertitle}

\author{\paperauthorA\thanks{Corresponding author: \corrEmail}
}
\author{\paperauthorB}
\affil{\affilOne}

\date{\today}

\newcommand{\paperabstract}[1]{%
  \begin{abstract}
    #1
  \end{abstract}
}
\newcommand{\paperkeywords}[1]{%
  \vspace{0.5em}
  \noindent\textbf{Keywords: } #1
}

\usepackage[nameinlink,noabbrev]{cleveref}

\begin{document}
\maketitle
\paperabstract{
Suppose we are given an ordered sequence of independent data whose distribution changes $K$ times at unknown locations, for some unknown $K \geq 0$. In this paper, we study the problem of performing distribution-free inference on $K$. First, we show an impossibility result: any distribution-free upper confidence bound on $K$ must be trivial and uninformative. Then, using conformal $p$-values, and under only the assumption that the data segments induced by the changepoints are exchangeable (within themselves) and mutually independent, we construct a finite-sample valid lower confidence bound on $K$, which we call the Conformal LOwer bound on Changepoint Count (CLOCC). We show that CLOCC is the \emph{only} feasible way to provide a  lower bound on $K$ under the stated assumptions, a property we refer to as  its universality. We provide practical guidelines for choosing score functions that yield efficient and tight lower bounds.
We evaluate CLOCC in several synthetic and real-data experiments, where it provides informative lower bounds on $K$, demonstrating its practical applicability.
}

\paperkeywords{\keywordslist}

\section{Introduction}

In this paper, we study the problem of distribution-free inference on the number of changepoints in an ordered data sequence. The changepoint count is an important modeling parameter. First, in real-world applications, both the number and locations of the changes are typically unknown. Therefore, estimating the full changepoint configuration can be statistically and computationally challenging. Knowledge of the changepoint count $K$ considerably simplifies this problem, since one then only needs to identify the most appropriate $(K+1)$-segmentation of the timeline. Second, $K$ is itself a quantitative summary of the complexity of the underlying data-generating mechanism: larger values of $K$ indicate that more distinct regimes are needed to explain the observed sequence. Thus, reliable inference on $K$ can directly inform appropriate modeling assumptions. Finally, in some applications \citep{zhao2025gaussian,vagnoli2018ensemble}, an unexpectedly large changepoint count can be treated as an early indication of failure in the underlying system. Given this practical relevance of the changepoint count, we therefore aim to quantify uncertainty about $K$ in a distribution-free and finite-sample valid manner, without imposing restrictive modeling assumptions or relying on large-sample approximations.

\subsection{Problem setting}

Suppose that for some $n\in \N$, we observe a sequence of $\mathcal{X}$-valued random variables $\bX:=(X_1,\ldots,X_n)$. Likewise, we use $\bx=(x_1,\ldots,x_n)$ to denote a generic element of $\mathcal{X}^n$. We assume that there exist $K\in\{0,\ldots,n-1\}$ unknown changepoints (so that there is at least one changepoint) at locations $\xi_0:=0<\xi_1<\xi_2<\cdots<\xi_K<n=:\xi_{K+1}$ such that
\[
(X_{\xi_{k-1}+1},\ldots,X_{\xi_k})\sim\Pcal^{(k)},\qquad\text{for each } k\in[K+1].
\]
Throughout this paper, we write $[k]$ to denote the collection $\{1,\ldots,k\}$ for $k\in \N$, and we write $\bxi=(\xi_1,\ldots,\xi_K)$ for the vector of true changepoint locations. Here, $K$ denotes the number of changepoints and is our target quantity, while for each $k\in[K+1]$, $\Pcal^{(k)}$, supported on $\mathcal{X}^{\xi_k-\xi_{k-1}}$, denotes the distribution of the $k$th segment. Here  Throughout, we assume that each $\xi_k$ is a genuine distributional change, in the sense that the two adjacent segments $(X_{\xi_{k-1}+1},\ldots,X_{\xi_k})$ and $(X_{\xi_{k}+1},\ldots,X_{\xi_{k+1}})$  cannot be merged into a single exchangeable segment. We write the corresponding joint distribution of $\bX$ as $\Pcal=\prod_{k=1}^{K+1}\Pcal^{(k)}$. In line with the distribution-free perspective, we impose no structural assumptions on $\Pcal^{(1)},\ldots,\Pcal^{(K+1)}$ beyond the following.

\begin{assumption}\label{ass:exchangeability}
For each $k\in[K+1]$, the segment $(X_{\xi_{k-1}+1},\ldots,X_{\xi_k})$ is exchangeable; that is, for any permutation $\pi:\{\xi_{k-1}+1,\ldots,\xi_k\}\mapsto\{\xi_{k-1}+1,\ldots,\xi_k\}$,
\[
(X_{\xi_{k-1}+1},\ldots,X_{\xi_k})\overset{d}{=}(X_{\pi(\xi_{k-1}+1)},\ldots,X_{\pi(\xi_k)}).
\]
Also, if $K \geq 1$, the segments $(X_{\xi_{k-1}+1},\ldots,X_{\xi_k})$ are mutually independent across $k\in[K+1]$. 
\end{assumption}
Note that when $K=0$, this assumption simply states that the whole data sequence $\bX$ is exchangeable. In general, for a simpler interpretation, one may consider the canonical setting, which we call the \emph{piecewise i.i.d.\ model}, where observations within the $k$th segment are i.i.d.\ from some distribution $P_k$, with $P_k\neq P_{k+1}$ for $k\in [K]$.

It is worth emphasizing that, apart from segmentwise exchangeability and mutual independence, we assume no further knowledge of the underlying distributions. In particular, we impose no additional structure on the observation space $\mathcal{X}$ or on the segmentwise distributions $\Pcal^{(1)},\ldots,\Pcal^{(K+1)}$, and, more importantly, we do not know either the number of changepoints or their true locations. This distributional agnosticism presents the main challenge in drawing meaningful and statistically valid inference on $K$.

\subsection{The goal}
\label{sec:defn_confidence_bounds}

We seek lower and upper bounds on $K$ with provable confidence guarantees. For each $K\in\{1,\ldots,n-1\}$, let $\mathfrak{P}_{K}$ denote the collection of joint distributions of $\bX$ that satisfy \Cref{ass:exchangeability} for some changepoint locations $0<\xi_1<\xi_2<\cdots<\xi_K<n$. Similarly, $\mathfrak{P}_0$ be the collection of joint distributions of $\bX$ such that $\bX$ is exchangaeble.

\begin{definition}[Distribution-free confidence bounds]\label{def:DF_bounds}
For any $\alpha\in(0,1)$, a statistic $\hat{L}_\alpha:\mathcal{X}^n\times[0,1]\mapsto\{0,\ldots,n-1\}$ is said to form a distribution-free lower confidence bound (\dflcb) on $K$ at level $1-\alpha$ if, for every $K\in\{0,\ldots,n-1\}$ and every $P\in\mathfrak{P}_K$,
\[
\P_{(\bX,\zeta)\sim P\times\textnormal{Unif}[0,1]}\bigl(K\ge\hat{L}_\alpha(\bX,\zeta)\bigr)\ge1-\alpha.
\]
Similarly, a statistic $\hat{U}_\alpha:\mathcal{X}^n\times[0,1]\mapsto\{0,\ldots,n-1\}$ is said to form a distribution-free upper confidence bound (\dfucb) on $K$ at level $1-\alpha$ if, for every $K\in\{0,\ldots,n-1\}$ and every $P\in\mathfrak{P}_K$,
\[
\P_{(\bX,\zeta)\sim P\times\textnormal{Unif}[0,1]}\bigl(K\le\hat{U}_\alpha(\bX,\zeta)\bigr)\ge1-\alpha.
\]
\end{definition}

In these definitions, the statistics $\hat{L}_\alpha$ and $\hat{U}_\alpha$ are functions of the observed data $\bX$, together with a random seed $\zeta\sim\textnormal{Unif}[0,1]$ that allows for internal randomization in their construction, if desired.

\subsection{Related work}\label{sec:review}

Retrospective changepoint analysis has a long history in statistics and signal processing; see \citet{truong2020selective,duggins2010parametric,niu2016multiple} for surveys. Broadly, one may distinguish three related problems. First, \emph{changepoint detection} asks whether \emph{any} distributional change is present. Second, when the number of changepoints is known, \emph{changepoint localization} aims to estimate their locations. Third, when both the number and locations of the changepoints are unknown, they must be estimated jointly.


\medskip\noindent\textbf{Changepoint detection methods.}
Changepoint detection is a classical and well-studied problem. CUSUM \citep{page1955test} is one of the earliest procedures for detecting a single changepoint and has inspired a large subsequent literature. Modern methods also address multiple changepoints and often provide estimates of their locations; examples include ordered multiple-changepoint tests \citep{aly2003tests}, kernel changepoint procedures (KCP) \citep{arlot2019kernel}, ChangeForest \citep{londschien2023random}, and graph-based methods \citep{chu2019asymptotic}. Separately, conformal martingale methods \citep{vovk2003testing,volkhonskiy2017inductive,shin2022detectors} provide powerful tools for online changepoint detection.

\medskip\noindent\textbf{Segmentation-based approaches.}
A common strategy for multiple-changepoint problems is to partition the timeline into approximately homogeneous segments and estimate the corresponding segment boundaries. These methods either optimize a global segmentation criterion, or take a recursive approach by repeatedly isolating individual changes. Among them, wild binary segmentation \citep{fryzlewicz2014wild}, and isolate-detect procedures \citep{anastasiou2022detecting} are particularly popular. For further examples, see \citet{fang2020segmentation,maidstone2017optimal}.

\medskip\noindent\textbf{Parametric localization methods.}
A substantial literature studies changepoint estimation under parametric assumptions. Classical likelihood-based methods derive estimators and uncertainty quantification under specific models; see, for example, \citet{kim1989likelihood,quandt1958estimation} for a single changepoint. In the multiple-changepoint setting, \citet{bai1998estimating,bai2003computation} developed estimation and testing procedures for multiple structural changes in linear models.

\medskip\noindent\textbf{Resampling-based approaches.}
Bootstrap and related resampling methods \citep{pettitt1979non,ross2012two} have also been used to quantify uncertainty in estimated changepoint locations. \citet{cho2022bootstrap} develop bootstrap intervals for multiple changepoints. These typically lack finite sample guarantees and are computationally intensive.

\medskip\noindent\textbf{Changepoint count estimation.}
\citet{yao1988estimating} proposed estimating the changepoint count using the Schwarz criterion. More generally, since the changepoint count determines the complexity of a segmentation, many approaches estimate it by optimizing a likelihood or contrast together with a complexity penalty \citep{lavielle2005using,killick2012optimal,harchaoui2010multiple}. These penalties are closely related to classical model-selection criteria such as AIC and BIC.
Focusing specifically on the distribution-free objective of our work, conformal prediction, originally introduced by \citet{vovk1999machine,shafer2008tutorial}, provides a general framework for distribution-free predictive inference. 

\medskip\noindent\textbf{Conformal approaches to changepoint localization.}
In the single-changepoint setting, \citet{dandapanthula2025conformal} introduced MCP localization, while \citet{hore2026conformal} proposed \conch{}, both of which construct confidence sets for the changepoint location with user-specified coverage. Building on \conch{}, \citet{hore2026distribution} studied a multi-stream setting and introduced \textsc{CROC}, which constructs a confidence set for the stream exhibiting the earliest change, viewed as a proxy for the ``root stream.'' \citet{yu2026conformal} recently extended \conch{} and \textsc{CROC} to settings with corrupted observations.

\subsection{Our approach}
We tackle the distribution-free inference on $K$ in two parts: first, by investigating a \dfucb on $K$, and then a \dflcb.

\medskip
\noindent\textbf{Impossibility of \dfucb:} We first show that any valid \dfucb on $K$ must be trivial and non-informative, in the sense that a \dfucb valid at confidence level $1-\alpha$ takes the maximum value $n-1$ with probability at least $1-\alpha$. 

\medskip
\noindent\textbf{An efficient \dflcb construction:}
We revisit the \conch{} framework of \cite{hore2026conformal}, which, in the setting $K=1$, constructs a distribution-free confidence set for the single changepoint $\xi_1\in[n-1]$. We then adapt its conformal $p$-value construction to build a \dflcb on the number of changepoints $K$. We call the resulting method Conformal LOwer bound on Changepoint Count (\clocc), which 
enjoys two important guarantees:

\begin{enumerate}
    \item \textbf{Finite-sample validity.} For any sample size $n$, any data-generating distribution $\Pcal$ satisfying \Cref{ass:exchangeability}, and any user-specified confidence level $1-\alpha\in(0,1)$, \clocc{} lower bounds the true changepoint count $K$ with probability at least $1-\alpha$.

    \item \textbf{Universality.} \clocc{} is universal for distribution-free lower confidence bounds on the changepoint count: any valid \dflcb can be represented as an instance of the \clocc{} framework through an appropriate choice of score function.

\end{enumerate}

Beyond these guarantees, we develop several practical variants of \clocc{} aimed at improving its practicality, and provide guidance for constructing informative CPP scores that can yield substantially tighter lower bounds in practice.

\paragraph{Organization of the paper.} The rest of the paper is organized as follows. In Section~\ref{sec:ucb_impossibility}, we prove the impossibility of a non-trivial \dfucb construction. Then, in Section~\ref{sec:method}, we formally introduce the \clocc{} algorithm, with its more practical variants presented in Section~\ref{sec:practical_clocc}. Next, in Section~\ref{sec:CPP_score}, we provide practical and theoretical guidelines for constructing efficient CPP scores, and in Section~\ref{sec:clocc_seg}, we introduce a computationally efficient construction of the \clocc{} \dflcb. Finally, in Section~\ref{sec:experiments}, we empirically evaluate the proposed \dflcb in several synthetic and real-data experiments.

\section{Non-trivial DF-UCB on $K$ is impossible}
\label{sec:ucb_impossibility}

Although the definitions of a \dfucb and a \dflcb look symmetric at first glance, they are fundamentally different. Put informally, when there is a substantial change along the observed data stream, one may be able to confidently detect such a change; a \dflcb essentially answers how many such changes can be identified with confidence. In contrast, no matter how similar two adjacent data points may look, one cannot confidently assert that they arise from the same distribution without imposing further distributional structure. Thus, any valid \dfucb must always be trivial in the following sense.

\begin{theorem}\label{thm:ucb_impossibility}
Fix $\alpha\in(0,1)$ and $n\in\N$. Suppose $\Xcal$ contains at least two points. Let $\hat{U}_\alpha:\Xcal^n\times[0,1]\to\{0,\ldots,n-1\}$ be a valid DF-UCB on the changepoint count $K$ at level $1-\alpha$. Then, for every fixed $\bx\in\Xcal^n$,
\[
\P\left(\hat{U}_\alpha(\bx,\zeta)=n-1\right)\ge1-\alpha,
\]
where the probability is taken over the randomness of $\zeta$.
\end{theorem}

The proof of this result is deferred to the appendix. On the other hand, building a non-trivial \dflcb is quite possible. For instance, given $n$ data points, split the sequence into approximately two halves, $(X_1,\ldots,X_{\lfloor n/2\rfloor})$ and $(X_{\lfloor n/2\rfloor+1},\ldots,X_n)$. We can run a distribution-free algorithm to test exchangeability separately on each of these two segments, controlling the Type~I error at level $\alpha/2$ for each test. Then, the number of tests that result in a rejection gives a non-trivial lower bound on the changepoint count.

While this provides a valid \dflcb, it is clearly inefficient, in the sense that it neither uses the Type~I error budget efficiently nor yields a tight lower bound. In later sections, we therefore improve upon this naive idea and focus on constructing efficient \dflcb on the changepoint count $K$.

\section{Conformal lower bound on changepoint count}\label{sec:method}

In this section, we construct a \dflcb on $K$, building upon the Conformal Changepoint Localization (\conch) algorithm introduced in \cite{hore2026conformal}. When there is a single changepoint in the data stream, \conch{} uses conformal $p$-values to construct a distribution-free confidence set for its location. Before introducing our construction, we briefly review the key components and the underlying ideas of \conch.

\subsection{Background:  \conch{} }
\label{sec:conch_review}

Suppose there is only a single changepoint $\xi\in[n-1]$, that is, $\bX\sim P$ for some $P\in\mathfrak{P}_1$. The goal of \conch{} is to localize the changepoint $\xi$ by returning a distribution-free confidence set for $\xi$. The \conch{} framework is built upon two key components: (1) a changepoint plausibility score and (2) a split-permutation group.

Any mapping $S:\mathcal{X}^n\times[n-1]\to\R$ is called a changepoint plausibility score, where $S(\bx,t)$ quantifies the plausibility that $t$ is the true changepoint location given a data sequence $\bx\in\Xcal^n$; a higher score indicates greater evidence that $t$ is the true changepoint.

Next, for each candidate location $t\in[n-1]$, they define the split-permutation group by
\[
\Pi_t:=\left\{\pi\in\mathcal S_n:\pi([t])=[t],\;\pi([n]\setminus[t])=[n]\setminus[t]\right\},
\]
i.e., the set of permutations that independently permute the indices to the left and right of $t$, without mixing them.

Given these two components, for each candidate changepoint index $t\in[n-1]$, \conch{} constructs the $p$-value
\[
p_t=\frac{1}{|\Pi_t|}\sum_{\pi\in\Pi_t}
\One{S(\pi(\bX),t)\le S(\bX,t)}.
\]
The $p$-value $p_t$ quantifies the evidence in $\bX$ against the null hypothesis $\tilde{\mathcal{H}}_{0,t}:\xi=t$, which posits that $t$ is the true changepoint location.

Under $\tilde{\mathcal{H}}_{0,t}$, the subsequence to the left of $t$ is exchangeable, and similarly the subsequence to the right of $t$ is exchangeable. Consequently, $p_t$ is super-uniform. Finally, by inverting the tests for $\tilde{\mathcal{H}}_{0,t}$ over $t\in[n-1]$, \conch{} produces
\[
\mathcal{C}_{n,1-\alpha}^{\conch}:=\{t\in[n-1]:p_t>\alpha\}
\]
as a distribution-free confidence set for the true changepoint location $\xi$.

While \conch{} provides a principled framework for localizing a single changepoint, inference on the number of changepoints presents two additional challenges. First, there may be multiple changepoints, and second, we do not know how many changepoints are present. In the following parts, we describe how to address both of these challenges.

\subsection{CLOCC: a conformal framework for \dflcb on $K$}
\label{sec:clocc_defn}

We now introduce Conformal LOwer bound on Changepoint Count (\clocc), which provides a valid \dflcb on $K$.
As mentioned earlier, we want to adapt \conch{} to a setting that allows multiple changepoints, with the number of changepoints itself being unknown. Therefore, we begin by generalizing the two core components of \conch{} as follows.

\medskip
\noindent\textbf{(1) ChangePoint Plausibility (CPP) score:}
For each $k\in[n-1]$, let
\[
\mathcal{T}_k:=\{(t_1,\ldots,t_k):1\le t_1<\cdots<t_k\le n-1\}
\]
denote the set of all possible $(k+1)$-segmentations of the full timeline $[n]$. Additionally, we define $\mathcal{T}_0:=\emptyset$ to indicate the case when there are no changepoints.

We write $\bar{\R}$ to denote the extended real, $\R\cup\{\infty,-\infty\}$, and call any mapping $S:\mathcal{X}^n\times\bigcup_{k=0}^{n-1}\mathcal{T}_k\to\bar{\R}$ a changepoint plausibility score function. In principle, for $\bt=(t_1,\ldots,t_k)\in\mathcal{T}_k$, the value $S(\bx,\bt)$ quantifies the plausibility that $\bt$ explains all the changepoint locations based on the observed data sequence $\bx$. Similarly, $S(\bx,\emptyset)$ quantifies the plausibility that there is no changepoint.

\medskip
\noindent\textbf{(2) Split-permutation group:}
For a fixed $k$-tuple $\bt=(t_1,\ldots,t_k)\in\mathcal{T}_k$, let us denote the induced $k+1$ segments by
\[
I_0=[1,t_1],\qquad I_\ell=[t_\ell+1,t_{\ell+1}] \ \text{for }\ell\in[k-1], \qquad I_k=[t_k+1,n].
\]
Let $\Pi_{\bt}$ denote the group of permutations of $[n]$ that independently permute within each segment $I_0,I_1,\ldots,I_k$, without mixing indices across these segments. Formally,
\[
\Pi_{\bt}=\left\{\pi\in\mathcal{S}_n:\text{ for every }t\in[n],\  t\in I_\ell \text{ implies }\pi(t)\in I_\ell \text{ for some }\ell\in\{0,\ldots,k\} \right\}.
\]
Further, for notational consistency, we write $\Pi_{\emptyset}$ to denote the full permutation group $\mathcal{S}_n$.

Given these two components, we can now follow the original \conch{} construction. For each $k\in \{0,1,\ldots,n-1\}$ and each $\bt\in\mathcal{T}_k$, define the conformal $p$-value
\begin{equation}
\label{eq:pvalue_conch_multi}
p_{\bt}:=\frac{1}{|\Pi_{\bt}|}\sum_{\pi\in\Pi_{\bt}}\One{S(\pi(\bX),\bt)\le S(\bX,\bt)}.
\end{equation}

For $k\in [n-1]$ and $\bt\in \mathcal{T}_k$, let $\tilde{\mathcal{H}}_{0,\bt}:\bxi=\bt$ denote the null hypothesis that there is at least a change and $\bt$ contains all the true changepoint locations. Then, $p_{\bt}$ quantifies the evidence in $\bX$ against $\tilde{\mathcal{H}}_{0,\bt}$. Under this null, each of the data segments induced by $\bt$ is exchangeable, and these segments are mutually independent. Consequently, the permutation construction above yields a $p$-value.

\begin{lemma}
\label{lem:validity_of_conch_multi}
Fix $\bt\in\bigcup_{k=1}^{n-1}\mathcal{T}_k$ and $\alpha\in(0,1)$. Under the null $\tilde{\mathcal{H}}_{0,\bt}$, the $p$-value $p_{\bt}$ in \eqref{eq:pvalue_conch_multi} satisfies
\[
\P\left(p_{\bt}\le\alpha\right)\le\alpha.
\]
\end{lemma}

On the other hand, $\tilde\Hcal_{0,\emptyset}$ only posits that the whole data sequence $\bX$ is exchangeable and there is no changepoint. Therefore, to test this null, we may simply define the conformal $p$-value
\begin{equation}\label{eq:pvalue_test_exchangeability}
   p_{\emptyset}:=\frac{1}{|\Pi_{\emptyset}|}\sum_{\pi\in\Pi_{\emptyset}}\One{S(\pi(\bX),\emptyset)\le S(\bX,\emptyset)}.
\end{equation}
By an analogous logic, $p_{\emptyset}$ is a super-uniform random variable under $\tilde\Hcal_{0,\emptyset}$.

Recall, however, that our objective is to draw inference on the number of changepoints $K$. Recall that the hypothesis of interest is $\mathcal{H}_{0,k}:K=k$, which asserts that the true number of changepoints is exactly $k$. 
Since $\mathcal{H}_{0,k}=\bigcup_{\bt\in\mathcal{T}_k}\tilde{\mathcal{H}}_{0,\bt}$ for $k\in\{0,\ldots,n-1\}$,
we define
\begin{equation}
\label{eq:pvalue_clocc}
p_{(k)}:=\max_{\bt\in\mathcal{T}_k}p_{\bt}
\end{equation}
as a $p$-value for $\mathcal{H}_{0,k}$, quantifying the evidence against the hypothesis that there are exactly $k$ changepoints.
Finally, we invert these tests to construct the following lower confidence bound for the changepoint count $K$:
\begin{equation}
\label{eq:clocc_dflcb}
\hat{L}_{\alpha}^{\clocc}(\bX):=\min\left\{k\in\{0,1,\ldots,n-1\}:p_{(k)}>\alpha\right\}.
\end{equation}
We note that if $K=n-1$, then all the corresponding segments are singletons, and hence $p_{(n-1)}\equiv 1$. Consequently, the set $\{k\in\{0,\ldots,n-1\}:p_{(k)}>\alpha\}$ is non-empty, and $\hat{L}_{\alpha}^{\clocc}(\bX)$ is always well-defined.
The procedure is formally summarized in \Cref{alg:clocc}.
The resulting $\hat{L}_{\alpha}^{\clocc}$ is a valid \dflcb on $K$, as we formally establish below.

\IncMargin{1.2em}
\begin{algorithm}[t]
    \caption{\clocc{}: conformal lower bound on changepoint count}
    \label{alg:clocc}
    \KwIn{$(X_t)_{t=1}^n$ (data); $1-\alpha$ (target coverage);
    $S:\mathcal X^n\times\bigcup_{k=0}^{n-1}\mathcal T_k\to\overline{\R}$ (CPP score)}
    \KwOut{$\hat L_{\alpha}^{\clocc}$ (\dflcb on the changepoint count)}

    \For{$k\in\{0,\ldots,n-1\}$}{
        $p_{(k)}\gets 0$\;

        \ForEach{$\bt\in\mathcal T_k$}{
            Construct the split-permutation group $\Pi_{\bt}$\;

            $p_{\bt}\gets\frac{1}{|\Pi_{\bt}|}
            \sum_{\pi\in\Pi_{\bt}}
            \One{
                S(\pi(\bX),\bt)\le S(\bX,\bt)
            }$\;

            $p_{(k)}\gets\max\{p_{(k)},p_{\bt}\}$\;
        }

        \If{$p_{(k)}>\alpha$}{
            $\hat L_{\alpha}^{\clocc}\gets k$\;
            \Return{$\hat L_{\alpha}^{\clocc}$}\;
        }
    }
\end{algorithm}
\DecMargin{1.2em}

\begin{theorem}[Validity of \clocc]
\label{thm:validity_clocc}
Fix $k\in\{0,\ldots,n-1\}$ and $\alpha\in(0,1)$. Under the null $\mathcal{H}_{0,k}$, the $p$-value $p_{(k)}$ satisfies $\P(p_{(k)}\le\alpha)\le\alpha$. Consequently, $\hat{L}^{\clocc}_\alpha$ is a distribution-free lower confidence bound for $K$ at level $1-\alpha$:
\[
\P_{X\sim \Pcal}\left(\hat{L}^{\clocc}_\alpha(\bX)\le K\right)\ge1-\alpha \qquad \text{for all}~P\in \mathfrak{P}_K.
\]
\end{theorem}

Notably, under the piecewise i.i.d.\ model, where observations within each segment are sampled i.i.d.\ from a common distribution, the $p$-value is valid not only under $\Hcal_{0,k}:K=k$, but also under the null $\Hcal_{0,\le k}:K\le k$. This follows because any refinement of the true segmentation preserves the exchangeability assumption within each resulting segment.

\subsection{Universality of \clocc}
In the earlier part, we established the \clocc{} framework as a principled way to construct a \dflcb on $K$ by exploiting the exchangeability structure in \Cref{ass:exchangeability}. Now, we show that this framework is in fact \emph{the only possible} way to build distribution-free lower confidence bounds on $K$.

\begin{theorem}\label{thm:universality_clocc}
Fix $\alpha\in(0,1)$. Let $L:\mathcal{X}^n\to\{0,\ldots,n-1\}$ be any procedure that gives a \dflcb on $K$ at level $1-\alpha$, i.e., for all $K\in \{0,1,\ldots,n-1\}$
\[
\P_{X\sim P}\bigl(L(\bX)\le K\bigr)\ge1-\alpha \qquad \text{for all}~P\in \mathfrak{P}_K.
\]
Then, there exists a score function $S:\mathcal{X}^n\times\bigcup_{k=0}^{n-1}\mathcal{T}_k\to\bar\R$ such that $L=\hat{L}^{\clocc}_\alpha$, where $\hat{L}^{\clocc}_\alpha$ is constructed using the CPP score $S$.
\end{theorem}

We refer to this result as the universality of the \clocc{} algorithm. It states that any procedure that constructs a valid \dflcb on $K$ must be an instance of \clocc{} with an appropriate choice of CPP score $S$. Such universality results are not new in the conformal literature. Their earliest appearances are in the predictive inference literature (see \citealp[Chapter~2.4]{vovk2005algorithmic}; \citealp[Theorem~9.6]{angelopoulos2024theoretical}), and more recently analogous results have been established in the context of changepoint analysis for \conch{} in \cite{hore2026conformal} and \textsc{CROC} in \cite{hore2026distribution}.

At first glance, the maximization over configuration-level $p$-values in~\eqref{eq:pvalue_clocc} may appear to be a potentially conservative way of constructing a $p$-value for $\mathcal{H}_{0,k}$. However, the crucial observation is: for each $k\in\{0,\ldots,n-1\}$,
\[
\mathcal{H}_{0,k}
=
\bigsqcup_{\bt\in\mathcal{T}_k}
\tilde{\mathcal{H}}_{0,\bt},
\]
so that $\mathcal{H}_{0,k}$ decomposes into disjoint null hypotheses corresponding to different changepoint configurations. The maximization step in~\eqref{eq:pvalue_clocc} therefore provides a natural valid $p$-value for this composite null without requiring any multiplicity correction and without compromsing efficiency. 

One consequence of this universality result is that, to construct an efficient \dflcb on $K$, we do not need to look beyond the \clocc{} framework; instead, we may focus only on choosing an informative score function $S$. Later, in Section~\ref{sec:CPP_score}, we discuss in more detail how the choice of CPP score affects the efficiency of \clocc{}.

Another consequence is that any heuristic procedure for estimating the changepoint count, that may or may not enjoy any theoretical guarantees, can be wrapped within the \clocc{} framework by suitably defining a CPP score to obtain a valid \dflcb.

\section{Practical implementations of \clocc}\label{sec:practical_clocc}

While the \clocc{} algorithm provides a universal framework for drawing inference on the changepoint count, several practical modifications can enhance its applicability. We discuss these variants one by one, keeping the original \clocc{} algorithm as the baseline. None of these variants compromise the core statistical guarantee of \clocc.

\subsection{\clocc-exact: towards exact validity of $p_{\bt}$}

Recall that the $p$-value $p_{\bt}$ in~\eqref{eq:pvalue_conch_multi} is only guaranteed to be super-uniform. In principle, this may lead to a somewhat loose \dflcb $\hat{L}_\alpha^{\clocc}$. While we do not observe a significant difference in our experiments, a simple randomization can make the $p$-value $p_{\bt}$ exactly uniform under the null. In particular, for any $\bt\in\bigcup_{k=0}^{n-1}\mathcal{T}_k$, define
\begin{equation}
\label{eq:pvalue_conch_multi_exact}
\bar{p}_{\bt}:=\frac{1}{|\Pi_{\bt}|}\left(\sum_{\pi\in\Pi_{\bt}}\One{S(\pi(\bX),\bt)<S(\bX,\bt)}+U\cdot\sum_{\pi\in\Pi_{\bt}}
\One{S(\pi(\bX),\bt)=S(\bX,\bt)}\right),
\end{equation}
where $U\sim\mathrm{Unif}[0,1]$ is generated independently of the data. This $p$-value quantifies the evidence against the null $\tilde{\mathcal{H}}_{0,\bt}$ and is exactly uniform under this null.

\begin{lemma}
\label{lem:validity_of_conch_multi_exact}
Fix $\bt\in\bigcup_{k=0}^{n-1}\mathcal{T}_k$ and $\alpha\in(0,1)$. Under the null $\tilde{\mathcal{H}}_{0,\bt}$, the $p$-value $\bar{p}_{\bt}$ in~\eqref{eq:pvalue_conch_multi_exact} satisfies
\[
\P\left(\bar{p}_{\bt}\le\alpha\right)=\alpha.
\]
\end{lemma}

Accordingly, let $\bar{p}_{(k)}:=\max_{\bt\in\mathcal{T}_k}\bar{p}_{\bt}$ for $k\in \{0,\ldots,n-1\}$ to finally construct the following lower bound on the changepoint count $K$:
\[
\hat{L}_{\alpha}^{\clocc\text{-exact}}(\bX):=\min\left\{k\in\{0,\ldots,n-1\}:\bar{p}_{(k)}>\alpha\right\}.
\]
Since these $p$-values are exactly uniform, $\bar{p}_{(n-1)}$ is no longer necessarily equal to $1$. Therefore, to ensure that $\hat{L}_{\alpha}^{\clocc\text{-exact}}(\bX)$ is well-defined, we set it equal to $n-1$ whenever the corresponding set is empty.
The resulting procedure is summarized in \Cref{alg:clocc_exact}.
By an argument analogous to that for \clocc{}, $\hat{L}_{\alpha}^{\clocc\text{-exact}}$ is a valid \dflcb on $K$.

\subsection{\clocc-MC: Monte--Carlo approximation to $p_{\bt}$}

Recall that for each $\bt\in\bigcup_{k=1}^{n-1}\mathcal{T}_k$, \clocc{} needs to compute the $p$-value $p_{\bt}$ as defined in~\eqref{eq:pvalue_conch_multi}. For $\bt=(t_1,\ldots,t_k)$ with $k\in[n-1]$, the corresponding split-permutation group $\Pi_{\bt}$ satisfies
\[
|\Pi_{\bt}|=t_1!(t_2-t_1)!\times\cdots\times(t_k-t_{k-1})!(n-t_k)!,
\]
which can be extremely large even for small $k$. Computing $p_{\bt}$ therefore requires evaluating $S(\pi(\bX),\bt)$ for every $\pi\in\Pi_{\bt}$, which can be expensive for a general score function $S$.

A natural solution is to compute a Monte--Carlo approximation to $p_{\bt}$. Fix $M\in\N$ and, independently of $\bX$, draw $\pi^{(1)},\ldots,\pi^{(M)}\overset{\mathrm{iid}}{\sim}\mathrm{Unif}(\Pi_{\bt})$, and define
\begin{equation}
\label{eq:mc_pvalue_conch_multi}
\widehat p_{\bt}:=\frac{1+\sum_{m=1}^M\One{S(\pi^{(m)}(\bX),\bt)\le S(\bX,\bt)}}{M+1}.
\end{equation}
The ``$+1$'' correction in both the numerator and denominator ensures finite-sample validity of the resulting Monte--Carlo $p$-value and is standard in the permutation-testing literature \citep{phipson2016permutation}. In fact, $\widehat p_{\bt}$ remains super-uniform under the null $\tilde{\mathcal{H}}_{0,\bt}$. Consequently, one may define $\widehat p_{(k)}:=\max_{\bt\in\mathcal{T}_k}\widehat p_{\bt}$ for $k\in \{0,\ldots,n-1\}$ and construct the \dflcb $\hat{L}_{\alpha}^{\clocc\text{-MC}}(\bX):=\min\{k\in\{0,\ldots,n-1\}:\widehat p_{(k)}>\alpha\}$, analogously to~\eqref{eq:clocc_dflcb}. We call this the \clocc-MC algorithm, given formally in~\Cref{alg:clocc_MC}.

\begin{theorem}
\label{thm:coverage-clocc_MC}
For any $\bt\in\bigcup_{k=1}^{n-1}\mathcal{T}_k$, $\widehat p_{\bt}$ defined in~\eqref{eq:mc_pvalue_conch_multi} is a valid $p$-value under $\tilde{\mathcal{H}}_{0,\bt}$. In particular, for any $\alpha\in(0,1)$,
\[
\P\bigl(\widehat p_{\bt}\le\alpha\bigr)\le\alpha.
\]
Consequently, $\hat{L}_{\alpha}^{\clocc\text{-MC}}$ is a valid \dflcb on $K$.
\end{theorem}

\subsection{\clocc-split: a split-conformal adaptation}

Drawing a parallel with the conformal predictive inference literature, \clocc{} can be viewed as a full-conformal adaptation to the task of lower bounding the number of changepoints. In particular, the CPP score function may depend arbitrarily on the observed data sequence. Thus, when computing the $p$-values in~\eqref{eq:pvalue_conch_multi}, one may need to relearn the score function for each permutation, which can add substantially to the computational burden.

To circumvent this cost, we propose a split-conformal adaptation of \clocc{}, where one subset of the observations is used to learn the CPP score and the remaining observations are used to compute the \clocc{} $p$-values.

Formally, let $\mathcal I_1:=\{i\in[n]:i\text{ is odd}\}$ and $\mathcal I_2:=[n]\setminus\mathcal I_1$ denote the odd and even indices, respectively, and write $\mathcal D_1=(X_i)_{i\in\mathcal I_1}$ and $\mathcal D_2=(X_i)_{i\in\mathcal I_2}$. We further write $\boldsymbol Y=(X_2,X_4,\ldots,X_{2\lfloor n/2\rfloor})$.

Using only $\mathcal D_1$, we first construct a CPP score $\widehat S=\mathcal A(\mathcal D_1):\mathcal X^{\lfloor n/2\rfloor}\times\bigcup_{k=1}^{\lfloor n/2\rfloor-1}\mathcal T_k\to\bar\R$. The learned score $\widehat S$ is then held fixed while running \clocc{} on $\boldsymbol Y$, yielding the \dflcb $\hat{L}_\alpha^{\clocc\text{-split}}$. We call this the \clocc-split algorithm, given formally in~\Cref{alg:clocc_split}.

\begin{theorem}
\label{thm:coverage-clocc_split}
The statistic $\hat{L}_{\alpha}^{\clocc\text{-split}}$ is a valid \dflcb on $K$.
\end{theorem}

Conditional on $\mathcal{D}_1$, the validity of $\hat{L}_{\alpha}^{\clocc\text{-split}}$ is immediate from Theorem~\ref{thm:validity_clocc}, and hence the proof is omitted.

While such sample splitting may lose some the changepoints, the number of true changepoints in the subsequence $\boldsymbol Y$ is always no larger than the number of true changepoints in the full data sequence $\bX$. Consequently, $\hat{L}_{\alpha}^{\clocc\text{-split}}$ remains a valid \dflcb for the original problem. In fact, if
\[
\max_{k\in [K+1]} (\xi_{k}-\xi_{k-1}) \ge 2,
\]
that is, no two true changepoints in $\bX$ are adjacent and changepoints are away from the boundaries, then no changepoint is lost under this interlaced sample splitting. Such non-adjacency is quite natural in most real world applications.

\section{Role of CPP score}
\label{sec:CPP_score}
In Section~\ref{sec:clocc_defn}, we established that the finite-sample validity of the \dflcb returned by the \clocc{} framework holds irrespective of the choice of CPP score function $S$. However, the CPP score plays an important role in efficiency: an informative score can result in a substantially tighter \dflcb. Note that the tightness of a lower bound is meaningful only when there is at least one changepoint. In this section, we therefore assume $K\ge 1$ and discuss the role of the CPP score in determining the efficiency of \clocc.

Recall that, within the \clocc{} framework, we first test the null $\tilde{\mathcal{H}}_{0,\bt}:\bxi=\bt$ via the $p$-value $p_{\bt}$ for each $\bt\in\bigcup_{k=1}^{n-1}\mathcal{T}_k$. The $p$-value $p_{\bt}$ is an adaptation of the \conch{} $p$-value to the multiple-changepoint setting, a construction that was also studied in the context of root-cause analysis in \cite{hore2026distribution}. We can therefore adapt Proposition~5.1 therein to obtain the following basic guidelines for choosing the CPP score.

\begin{proposition}\label{prop:score-properties}
Fix $n\in\mathbb{N}$ and $\alpha\in(0,1)$.
\begin{itemize}
    \item \textnormal{\textbf{(Symmetry yields power loss).}} Fix $k\in\{1,\ldots,n-1\}$ and $\bt\in\mathcal{T}_k$. If $S$ satisfies $S(\cdot,\bt)=S(\pi(\cdot),\bt)$ for all $\pi\in\Pi_{\bt}$, then the $p$-value $p_{\bt}$ in~\eqref{eq:pvalue_conch_multi} equals $1$. Consequently, $p_{(k)}=1$ and $\hat{L}_{\alpha}^{\clocc}(\bx)\le k$.

    \item \textbf{(Conformal data-processing inequality).} Let $L_1$ be the \dflcb returned by \clocc{} using the score $S$. For any non-decreasing function $f:\bar\R\to\bar\R$, let $L_2$ be the corresponding \dflcb obtained using $f(S)$. Then $L_1\ge L_2$, with equality whenever $f$ is strictly increasing.
\end{itemize}
\end{proposition}

\subsection{Optimal CPP score}

To characterize an optimal choice of score, we consider an oracle piecewise i.i.d.\ changepoint model specified by a sequence of distributions $P_1,P_2,\ldots$, where each $P_j$ admits a density $f_j$ with respect to a common dominating measure $\mu$.  Specifically, when there are $K$ changepoints at locations $\bxi=(\xi_1,\ldots,\xi_K)$, with $\xi_0=0$ and $\xi_{K+1}=n$, we assume
\[
\Pcal^{(j)}=(P_j)^{\xi_j-\xi_{j-1}},\qquad j\in[K+1].
\]
Thus, the observations in the $j$th segment are i.i.d.\ from $P_j$. Let $\Pcal_{\mathrm{IID}}\subseteq\bigcup_{k=1}^{n-1}\mathfrak{P}_k$ denote the subclass of distributions arising under this model.

Here for the optimality analysis, we resort to the \clocc-exact algorithm, since it is less conservative than the original \clocc, and facilitates a fair analysis. From the last section, the core component of \clocc-exact, the conformal $p$-value $\bar p_{\bt}$ in~\eqref{eq:pvalue_conch_multi_exact} tests the null $\tilde{\mathcal{H}}_{0,\bt}$, and only the component $S(\cdot,\bt)$ is relevant for this test.  Given knowledge of the true changepoint vector $\bxi$, designing an efficient CPP score can therefore be viewed through the problem of building a powerful test for $\tilde{\mathcal{H}}_{0,\bt}$ versus $\tilde{\mathcal{H}}_{0,\bxi}$.

For this analysis, we assume the knowledge of the true number of changepoints $K$, and consider candidate vectors $\bt=(t_1,\ldots,t_\ell)\in\bigcup_{\ell=1}^{n-1}\mathcal{T}_\ell$, with the conventions $t_0=\xi_0=0$ and $t_{\ell+1}=\xi_{K+1}=n$. We will use the exact conformal $p$-value $(\bar p_\bt)$, defined in~\eqref{eq:pvalue_conch_multi_exact} for the aforementioned testing problem. As a by-product of the construction, the corresponding $p$-values also yield a confidence set for the full changepoint vector,
\begin{equation}\label{eq:conch_multi}
    \bar\Ccal_{1-\alpha}^{\conch\text{-multi}}
    :=\left\{\bt\in\bigcup_{\ell=1}^{n-1}\mathcal{T}_\ell:\bar p_{\bt}>\alpha\right\}.
\end{equation}
An efficient CPP score should ideally result in a small $\Ccal_{1-\alpha}^{\conch\text{-multi}}$. The following result characterizes an optimal score for this purpose.

\begin{theorem}\label{thm:optimal_score}
Fix $\bt=(t_1,\ldots,t_\ell)$ and $\bxi=(\xi_1,\ldots,\xi_K)$. Any strictly increasing transformation of the CPP score $S^{\mathrm{OPT}}$ defined by
\begin{equation}\label{eq:optimal_CPP_score}
S^{\mathrm{OPT}}(\bx,\bt)
=\log\left(\frac{\prod_{j=1}^{\ell+1}\prod_{i=t_{j-1}+1}^{t_j}f_j(x_i)}{\prod_{j=1}^{K+1}\prod_{i=\xi_{j-1}+1}^{\xi_j}f_j(x_i)}\right)
\end{equation}
is optimal, i.e., for any $\bxi\in\mathcal{T}_K$, any strictly increasing function $f:\bar\R\to\bar\R$ and any score $S:\mathcal{X}^n\times\bigcup_{\ell=1}^{n-1}\mathcal{T}_\ell\to\bar\R$,
\[
\E_{\tilde{\Hcal}_{0,\bxi}\,\cap\,\Pcal_{\mathrm{IID}}}\!\left[|\bar{\Ccal}_{1-\alpha}^{\conch\text{-multi}}(S)|\right]\ge\E_{\tilde{\Hcal}_{0,\bxi}\,\cap\,\Pcal_{\mathrm{IID}}}\!\left[|\bar{\Ccal}_{1-\alpha}^{\conch\text{-multi}}(f(S^{\mathrm{OPT}}))|\right].
\]
\end{theorem}

This result extends the expression of the optimal score function for \conch{} in the single-changepoint setting (cf.\ Section~4 of \cite{hore2026conformal}).

\subsection{Practical CPP scores}

The optimal score in~\eqref{eq:optimal_CPP_score} depends on the number of changepoints $K$, the true changepoint locations $\bxi$, and the unknown densities $f_1,f_2,\ldots$. Therefore, the optimal score cannot be used directly. In practice, we instead mimic its main structure, namely, comparing how well a candidate segmentation $\bt$ explains the observed sequence relative to the ``best'' reference segmentation.

First, one may use any changepoint localization algorithm to obtain a reference estimate
\[
\hat{\bxi}:=(\hat\xi_1,\ldots,\hat\xi_{\hat K}).
\]
One can use any off-the-shelf method to obtain this estimate; a theoretically grounded approach would be to use kernel changepoint detection (KCPD) from \citet{arlot2019kernel,garreau2018consistent}. Here, $\hat K$ is the estimated changepoint count. The estimates $\hat{\bxi}$ and $\hat K$ are then used as proxies for $\bxi$ and $K$, respectively. In particular, when the segment densities $f_1,f_2,\ldots$ are known, we may define
\begin{equation}\label{eq:oracle_likelihood}
L(\bx,\bt)=\sum_{j=1}^{k+1} \sum_{i=t_{j-1}+1}^{t_j} \log(f_j(x_i)),
\end{equation}
for any $(t_1,\ldots,t_k)\in \mathcal{T}_k$ and $k\in \{1,\ldots,n-1\}$, as the complete log-likelihood of the data sequence $\bx$ under a candidate segmentation $\bt$. We then define the oracle log-likelihood ratio (oracle LLR) score
\begin{equation}\label{eq:oracle_CPP_score}
S^{\mathrm{orcl}}(\bx,\bt)
=\log\left(\frac{\prod_{j=1}^{\ell+1}\prod_{i=t_{j-1}+1}^{t_j}f_j(x_i)}{\prod_{j=1}^{\hat K+1}\prod_{i=\hat\xi_{j-1}+1}^{\hat\xi_j}f_j(x_i)}
\right).
\end{equation}

However, in practice, the true densities $f_1,f_2,\ldots$ are unknown, and therefore the score $S^{\mathrm{orcl}}$ cannot be computed. Instead, we may obtain estimates $\hat f_1,\hat f_2,\ldots$ and consider the score
\begin{equation}\label{eq:learned_CPP_score}
S(\bx,\bt)=\log\left(
\frac{\prod_{j=1}^{\ell+1}\prod_{i=t_{j-1}+1}^{t_j}\hat f_j(x_i)}{\prod_{j=1}^{\hat K+1}\prod_{i=\hat\xi_{j-1}+1}^{\hat\xi_j}\hat f_j(x_i)}
\right).
\end{equation}
These density estimates may be obtained parametrically or nonparametrically. When the score itself is learned from the observed data, the full \clocc{} procedure requires relearning these components for every permutation to ensure validity.

A computationally cheaper alternative is to implement the \clocc-split algorithm using interlaced sample splitting: the reference segmentation, density estimates, or classifier-based likelihood ratios can be learned on the first split and then held fixed while computing the conformal $p$-values on the second split. It is worth noting that the numerator in the score $S(\bx,\bt)$ defined in~\eqref{eq:learned_CPP_score} is invariant under permutations in $\Pi_\bt$. Therefore, to compute the $p$-value $p_\bt$, one may equivalently use $-\hat{L}(\bx;\hat{\bxi})$ as the CPP score, where
\begin{equation}\label{eq:likelihood_at_hatxi}
\hat{L}(\bx;\hat{\bxi}):=\prod_{j=1}^{\hat K+1}\prod_{i=\hat\xi_{j-1}+1}^{\hat\xi_j}\hat f_j(x_i)
\end{equation}
is the complete estimated likelihood of $\bx$ under the reference segmentation $\hat{\bxi}$.

Direct density estimation can be difficult in high-dimensional or structured data settings. However, given a candidate segmentation $\bt=(t_1,\ldots,t_\ell)$, the oracle CPP score in~\eqref{eq:oracle_CPP_score} can be equivalently written as
\[
S^{\mathrm{orcl}}(\bx,\bt)=\sum_{j=2}^{\ell+1}\sum_{i=t_{j-1}+1}^{t_j}\log\left(\frac{f_j}{f_1}(x_i)\right)
-\sum_{j=2}^{\hat K+1}\sum_{i=\hat\xi_{j-1}+1}^{\hat\xi_j}\log\left(\frac{f_j}{f_1}(x_i)\right).
\]
This classifier-based implementation provides a particularly convenient way to compute the \clocc-split \dflcb. We summarize the resulting practical procedure in \Cref{alg:clocc_split_classification}.

Thus, instead of estimating the densities directly, one may estimate the relevant density ratios via multiclass classification: after assigning observations to the estimated segments, a multiclass classifier can be trained to distinguish among the segment distributions, and the resulting logits can be used as proxies for the corresponding log-density ratios.

Importantly, none of these modeling choices affect the finite-sample validity of \clocc{}; they only affect the tightness of the resulting \dflcb.




\section{\clocc-SEG: an efficient segmentwise test}
\label{sec:clocc_seg}

While \clocc-MC and \clocc-split simplify the computation of an individual conformal $p$-value, in the worst case, one still needs to compute $p_{\bt}$ for every $\bt\in\bigcup_{k=1}^{n-1}\mathcal{T}_k$. The total number of such candidate changepoint vectors is
\[
\sum_{k=1}^{n-1}|\mathcal{T}_k|=
\sum_{k=1}^{n-1}\binom{n-1}{k}=2^{n-1}-1.
\]
Thus, even with Monte--Carlo approximation and sample splitting, evaluating a separate $p$-value $p_{\bt}$ for every feasible $\bt$ can become computationally prohibitive for large $n$. It is worth noting that this is the worst-case scenario. In practice, to compute \dflcb $\hat{L}_\alpha^{\clocc}$, since it is the smallest $k$ such that $p_{(k)}$ in~\eqref{eq:pvalue_clocc} is larger than the nominal level $\alpha$, we may start the search with $k=1$ and stop at the smallest $k$ where we find success. Since, in practice, $K$ is much smaller than $n$, we would need to enumerate a much smaller collection of candidate changepoints.

Another practical modification to reduce this computational cost is to restrict the candidate changepoint locations to a pre-fixed feasible set (typically smaller). In many applications, changepoints are known a priori to lie on a prespecified grid $\mathcal{G}\subseteq[n-1]$. In this case, one may replace $\mathcal{T}_k$ by
\[
\mathcal{T}_k(\mathcal{G})
:=
\left\{(t_1,\ldots,t_k)\in\mathcal{T}_k:t_j\in\mathcal{G}\text{ for all }j\in[k]\right\},
\]
and define $p_{(k)}$ by computing the maximum of $p_{\bt}$ only over $\bt\in\mathcal{T}_k(\mathcal{G})$. Provided that the true changepoint vector belongs to $\mathcal{G}$, the validity result still holds while the number of candidate segmentations can be substantially reduced. We use such prespecified grids in some of our numerical experiments.

Even with these modifications, we can not fully eliminate its combinatorial cost when no sufficiently small feasible grid is available. We therefore develop a more systematic scalable variant, called \clocc-SEG, which avoids computing a separate permutation $p$-value for every candidate changepoint configuration.

Fix $k\in[n-1]$ and a candidate changepoint vector $\bt=(t_1,\ldots,t_k)\in\mathcal{T}_k$, with $t_0=0$ and $t_{k+1}=n$. The vector $\bt$ induces the segmentation
\[
I_j(\bt):=[t_{j-1}+1,t_j],
\qquad j\in[k+1].
\]
For any interval $I=[s,t]$ with $1\le s\le t\le n$, let $\mathcal{H}^{\mathrm{ex}}_{0,I}$ denote the null hypothesis that $(X_s,\ldots,X_t)$ is exchangeable. Under $\tilde{\mathcal{H}}_{0,\bt}$, each of the induced segments is exchangeable, and hence
\[
\tilde{\mathcal{H}}_{0,\bt}
\subseteq
\bigcap_{j=1}^{k+1}
\mathcal{H}^{\mathrm{ex}}_{0,I_j(\bt)}.
\]
Therefore, instead of directly testing $\tilde{\mathcal{H}}_{0,\bt}$ through a joint permutation over all induced segments, one may test exchangeability separately within each $I_j(\bt)$ and then combine the resulting $p$-values to obtain a single valid $p$-value for $\tilde{\mathcal{H}}_{0,\bt}$.

For any $1\le s<t\le n-1$, let $\mathcal{S}_{[s,t]}$ denote the set of all permutations of the indices in $[s,t]$. Given a score (that is preferably not symmetric in its arguments) $A_{s,t}:\mathcal{X}^{t-s+1}\to\mathbb{R}$, we may test $\mathcal{H}^{\mathrm{ex}}_{0,[s,t]}$ using the conformal $p$-value
\[
p_{s,t}
:=
\frac{1}{|\mathcal{S}_{[s,t]}|}
\sum_{\pi\in\mathcal{S}_{[s,t]}}
\One{
A_{s,t}(X_{\pi(s)},\ldots,X_{\pi(t)})
\le
A_{s,t}(X_s,\ldots,X_t)
}.
\]
Under $\mathcal{H}^{\mathrm{ex}}_{0,[s,t]}$, $p_{s,t}$ is super-uniform. As before, this full-permutation $p$-value may be replaced by its Monte--Carlo analogue in practice.

While any suitable score $A_{s,t}$ can be used for this exchangeability test, a natural choice can often be obtained directly from the CPP scores developed for the original \clocc{} algorithm. In particular, many CPP scores admit an additive decomposition of the form
\[
S(\bx,\bt)
=
\sum_{i=1}^n S_i(\bx,\bt),
\]
where $S_i(\bx,\bt)$ denotes the contribution of the $i$th observation to the overall CPP score. The LLR scores introduced in the previous section admit exactly this form. Consequently, for testing exchangeability within an induced segment $I_j(\bt)$, a natural segmentwise score is obtained by restricting this sum to the corresponding indices, i.e., $\sum_{i=t_{j-1}+1}^{t_j} S_i(\bx,\bt)$.
Whenever this restricted score depends only on the observations within the corresponding interval, it can be used as $A_{s,t}$. Thus, \clocc-SEG can retain the same underlying score information while replacing the global permutation test by separate exchangeability tests on the constituent segments.

Now consider again $\bt=(t_1,\ldots,t_k)\in\mathcal{T}_k$. Under $\tilde{\mathcal{H}}_{0,\bt}$, the $k+1$ induced segments are mutually independent by \Cref{ass:exchangeability}. Hence, provided that the intervalwise $p$-values depend only on the observations within their respective intervals, the segmentwise $p$-values
\[
p_{t_0+1,t_1},\,
p_{t_1+1,t_2},\,
\ldots,\,
p_{t_k+1,t_{k+1}}
\]
are mutually independent under the null. We may therefore combine them using Fisher's rule and define
\begin{equation}
\label{eq:clocc_seg_pval}
p_{\bt}^{\mathrm{seg}}
:=
\overline F_{\chi^2_{2(k+1)}}
\left(
-2\sum_{j=1}^{k+1}
\log p_{t_{j-1}+1,t_j}
\right),
\end{equation}
where $\overline F_{\chi^2_d}(x):=\P(\chi^2_d\ge x)$ denotes the survival function of a chi-squared random variable with $d$ degrees of freedom. Consequently, $p_{\bt}^{\mathrm{seg}}$ is itself a valid $p$-value under $\tilde{\mathcal{H}}_{0,\bt}$.

We then proceed exactly as in the original \clocc{} construction by defining
\begin{equation}
\label{eq:clocc_seg_pk}
p_{(k)}^{\mathrm{seg}}
:=\max_{\bt\in\mathcal{T}_k}p_{\bt}^{\mathrm{seg}},\qquad
\hat L_{\alpha}^{\clocc\text{-SEG}}:=
\min\left\{k\in[n-1]:p_{(k)}^{\mathrm{seg}}>\alpha\right\}.
\end{equation}
The resulting procedure is summarized in \Cref{alg:clocc_seg}. By the same argument as in Theorem~\ref{thm:validity_clocc}, $\hat L_{\alpha}^{\clocc\text{-SEG}}$ is a valid \dflcb on $K$.

The key computational gain is that the expensive permutation step is now performed only at the interval level. Every segment appearing in any candidate changepoint vector is a contiguous interval $[s,t]\subseteq[n]$, and there are only $n(n+1)/2$ such intervals. Hence, the collection $\{p_{s,t}:1\le s\le t\le n\}$
can be computed and stored once and subsequently reused across all candidate segmentations. As a consequence, \clocc-SEG requires only $O(n^2)$ distinct intervalwise $p$-value computations, rather than a separate permutation computation for each of the $2^{n-1}$ candidate changepoint configurations.

Thus, \clocc-SEG retains the finite-sample validity of \clocc{} while reducing the number of distinct permutation $p$-value computations from exponential to quadratic in $n$.


\section{Experiments}\label{sec:experiments}
In this section, we evaluate the performance of the \clocc{} \dflcb in several synthetic and semi-synthetic experiments. For all experiments, we use the \clocc-split algorithm, with the $p$-values computed using the Monte--Carlo approximation in~\eqref{eq:mc_pvalue_conch_multi} with $M=200$. For all experiments in this section, we assume that all changepoints are known to lie on the prespecified grid
$\mathcal{G}:= 50\cdot \N$,
and restrict the candidate changepoint vectors accordingly, as described in Section~\ref{sec:clocc_seg}.

\subsection{Gaussian mean-shift experiment}
We start by evaluating the empirical coverage and tightness of the \clocc{} \dflcb in a Gaussian mean-shift setting. Specifically, we set the sample size to $n=1000$, and for a given $K$, there are changepoints at locations $1\le \xi_1<\ldots<\xi_K\le n-1$, sampled randomly from a pre-fixed grid of candidate changepoints $\mathcal{G}:=\{50,100,\ldots,950\}$. Moreover, for each $k\in\{1,\ldots,K+1\}$,
\[
X_{\xi_{k-1}+1},\ldots,X_{\xi_k}\overset{iid}{\sim} \mathcal{N}(\mu_k,1),
\]
where the segment means $\mu_k$ are chosen randomly from the grid $\{-1.75,-1,0,1,1.75\}$, while ensuring that adjacent segment means are different.

\paragraph{Experiment 1: Coverage and tightness in a hierarchical model.}
We vary $K\in \{2,3,4,5\}$ and, for each such choice, generate $\bX$ according to the Gaussian mean-shift model described above. For each run, we follow the interlaced sample-splitting approach in \clocc-split: the odd-indexed samples are used to compute the reference estimator $\hat{\bxi}$, and the segment densities are then learned either parametrically, assuming that the samples within each segment are generated from a Gaussian distribution with unknown mean and unit variance, or non-parametrically using kernel density estimation. The remaining samples are then used to compute the \clocc-MC $p$-values and obtain the final \dflcb.

We generate $25$ batches, each consisting of $25$ independently generated sequences $\bX$, and within each batch compute the empirical coverage, $\hat{\P}(K\ge \hat{L}(\bX))$, together with the empirical probabilities that $\hat L(\bX)$ equals $K$, $K-1$, $K-2$, $K+1$, or $K+2$. These later probabilities help us assess tightness of the resulting \dflcb.

Further, we repeat the same evaluation for a hierarchical model in which the changepoint count itself is random and, for each sequence, is sampled uniformly from $\{1,\ldots,10\}$.

Table~\ref{tab:clocc_coverage_tightness} reports the resulting empirical distribution of $\hat L_\alpha^{\clocc}$ across these different changepoint settings. As expected, the empirical coverage remains above the nominal level $1-\alpha=0.9$ in every setting, for both the parametric and nonparametric CPP scores. At the same time, the lower bound is quite tight: $\hat L_\alpha^{\clocc}$ equals the true changepoint count $K$ with high probability across all settings, and when it differs from $K$, the discrepancy is typically of one changepoint with very high probability. The results are similar for the hierarchical setting with random $K$.

\begin{table}[t]
\centering
\small
\setlength{\tabcolsep}{5.5pt}
\renewcommand{\arraystretch}{1.30}

\begin{tabular}{|l|c|c|c|c|c|}
\hline
& $K=2$ & $K=3$ & $K=4$ & $K=5$ & Random $K$ \\
\hline
\multicolumn{6}{|l|}{\textbf{Parametric CPP score}} \\
\hline
$\P(\hat L_\alpha\le K)$
& 0.971 (0.006) & 0.969 (0.008) & 0.984 (0.006) & 1.000 (0.000) & 0.980 (0.008) \\
\hline
$\P(\hat L_\alpha=K)$
& 0.955 (0.008) & 0.944 (0.010) & 0.945 (0.008) & 0.925 (0.009) & 0.933 (0.009) \\
\hline
$\P(\hat L_\alpha=K-1)$
& 0.016 (0.005) & 0.024 (0.005) & 0.037 (0.006) & 0.060 (0.008) & 0.044 (0.008) \\
\hline
$\P(\hat L_\alpha=K-2)$
& 0.000 (0.000) & 0.001 (0.001) & 0.001 (0.001) & 0.015 (0.005) & 0.003 (0.002) \\
\hline
$\P(\hat L_\alpha=K+1)$
& 0.029 (0.006) & 0.029 (0.008) & 0.016 (0.006) & 0.000 (0.000) & 0.019 (0.008) \\
\hline
$\P(\hat L_\alpha=K+2)$
& 0.000 (0.000) & 0.001 (0.001) & 0.000 (0.000) & 0.000 (0.000) & 0.001 (0.001) \\
\hline\hline
\multicolumn{6}{|l|}{\textbf{Nonparametric CPP score}} \\
\hline
$\P(\hat L_\alpha\le K)$
& 0.971 (0.006) & 0.967 (0.008) & 0.981 (0.005) & 1.000 (0.000) & 0.980 (0.007) \\
\hline
$\P(\hat L_\alpha=K)$
& 0.948 (0.010) & 0.931 (0.010) & 0.937 (0.009) & 0.920 (0.011) & 0.925 (0.009) \\
\hline
$\P(\hat L_\alpha=K-1)$
& 0.021 (0.006) & 0.035 (0.006) & 0.040 (0.006) & 0.065 (0.009) & 0.051 (0.007) \\
\hline
$\P(\hat L_\alpha=K-2)$
& 0.001 (0.001) & 0.001 (0.001) & 0.004 (0.002) & 0.013 (0.005) & 0.004 (0.002) \\
\hline
$\P(\hat L_\alpha=K+1)$
& 0.029 (0.006) & 0.032 (0.008) & 0.019 (0.005) & 0.000 (0.000) & 0.019 (0.006) \\
\hline
$\P(\hat L_\alpha=K+2)$
& 0.000 (0.000) & 0.001 (0.001) & 0.000 (0.000) & 0.000 (0.000) & 0.001 (0.001) \\
\hline
\end{tabular}

\vspace{0.7em}
\caption{Empirical distribution of $\hat L_\alpha^{\clocc}$ across different changepoint settings. Each entry reports the empirical probability, with its standard error in parentheses. The first row within each CPP-score setting reports the empirical coverage probability $\P(\hat L_\alpha^{\clocc}\le K)$, which remains above the target level $1-\alpha=0.9$ across all settings. The remaining rows describe the tightness of the lower bound. In particular, $\hat L_\alpha^{\clocc}$ equals the true changepoint count $K$ with high probability, while deviations from $K$ are typically small.}
\label{tab:clocc_coverage_tightness}
\end{table}


\paragraph{Experiment 2: Effect of signal strength on tightness.}
Next, we study how the tightness of the \clocc{} \dflcb varies with the strength of the distributional changes. We fix $K=6$ and the changepoint locations at $\{100, 300, 400, 500, 700, 900\}$ throughout the experiment. Starting from $\mu_1=1$, the successive segment means are generated according to
\[
\mu_{j+1}=\mu_j+\delta S_jA_j,\qquad j\in[K],
\]
where $S_j$ is sampled uniformly from $\{-1,+1\}$ and $A_j\sim\mathrm{Unif}(0.8,1.2)$ independently. In other words, each adjacent mean is equally likely to be smaller or larger than the previous one, with the magnitude of the change being approximately $\delta$.
We vary $\delta$ over a range of signal strengths in $\{0.1, 0.25, \ldots, 1.15, 1.30\}$ and follow the same experimental design as in Experiment~1 to evaluate \clocc. Figure~\ref{fig:clocc_signal_strength} reports the empirical coverage and the average tightness gap $K-\hat L_\alpha$ as functions of $\delta$. While coverage remains controlled across signal strengths, the lower confidence bound becomes progressively looser as $\delta$ decreases and the adjacent segments become harder to distinguish.

\begin{figure}[!h]
    \centering
    \includegraphics[width=\textwidth]{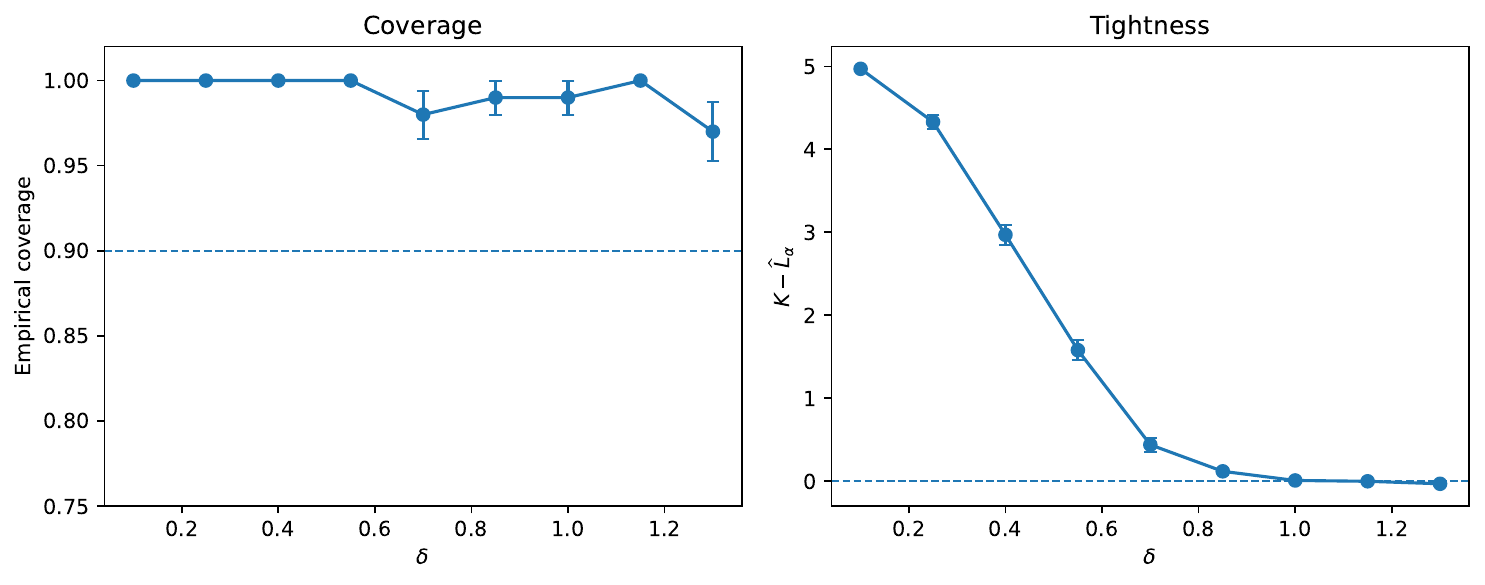}
    \caption{Separation between segment distributions dictates tightness of \clocc.
    Left: empirical coverage of the lower confidence bound, with the dashed horizontal line indicating the nominal level $1-\alpha=0.9$. Right: average tightness gap $K-\hat L_\alpha$. As $\delta$ increases, the lower confidence bound becomes progressively tighter and approaches the true changepoint count.}
    \label{fig:clocc_signal_strength}
\end{figure}

\subsection{MNIST digit-shift experiment}
Next, we consider a semi-synthetic multiple-changepoint setting based on images of handwritten digits from the MNIST dataset. We generate data sequences of length $n=1000$, with changepoints fixed at $\bxi=(200,400,600,800)$, so that the changepoint count is $K=4$. The five successive segments consist of i.i.d.\ samples from digits $1,7,3,8$, and $5$, respectively. The left panel of Figure~\ref{fig:mnist_clocc} illustrates one such data sequence, showing each of the four digit shifts.

Analogous to the previous Gaussian mean-shift setting, we use an interlaced sample split: the odd-indexed observations are first used to learn a reference estimator using kernel changepoint detection (KCPD), based on a $10$-dimensional representation learned from a neural network classifier. We then train another multiclass classifier to distinguish the estimated segments, and the logits of this classifier are used to construct the LLR scores. Finally, we apply \clocc-MC with confidence level $1-\alpha=0.9$ to compute the \dflcb. In the right panel of Figure~\ref{fig:mnist_clocc}, we report a bar plot of the resulting \dflcb $\hat L_\alpha$. \clocc{} successfully lower bounds the true changepoint count $K=4$ with probability $0.98$ and, in fact, takes the value $4$ exactly most of the time. This experiment illustrates the applicability of \clocc{} to generic high-dimensional data spaces, such as images.

\begin{figure}[!h]
    \centering
    \includegraphics[width=\textwidth]{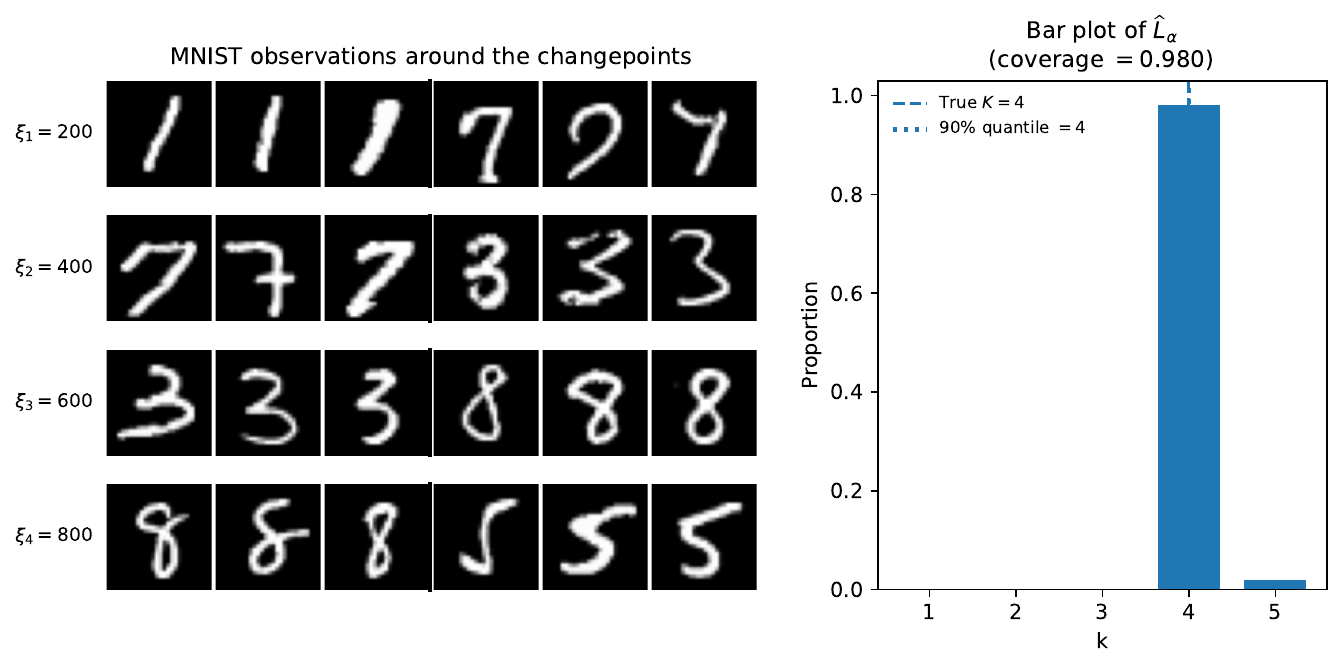}
    \caption{\clocc{} gives a tight lower bound on $K$ for the MNIST digit-shift experiment. On the left, we show representative samples immediately before and after each of the four true changepoints. On the right, we show a bar plot of the lower confidence bound $\hat L_\alpha$ over repeated experiments.}
    \label{fig:mnist_clocc}
\end{figure}

\subsection{SST-2 sentiment-shift experiment}
Finally, we consider a semi-synthetic changepoint experiment based on the SST-2 sentiment dataset, which consists of human-written reviews of books, movies, etc. Each observation here is high-dimensional text data. We generate data sequences of length $n=1500$ with changepoints at $(300,600,900,1200)$, so that there are $K=4$ true changepoints. Within each segment, observations are sampled i.i.d.\ from a mixture of positive and negative reviews. Across the five successive segments, the proportion of positive sentiment, denoted by $\delta$, takes the values
\[
(0.70,\;0.50,\;0.65,\;0.45,\;0.60),
\]
respectively. We follow the same strategy as in the earlier MNIST experiment to implement \clocc-split with Monte--Carlo $p$-values. The only change is that, to learn the reference estimator $\hat{\bxi}$ from the odd-indexed observations, we apply the KCPD algorithm to a two-dimensional representation of each review obtained from a pretrained DistilBERT sentiment classifier, which is well suited to textual data. Figure~\ref{fig:sst2_clocc} shows a bar plot of $\hat L_\alpha$. In contrast to MNIST, the neighboring regimes here have substantial overlap in their sentiment distributions, making this experiment more challenging. Nevertheless, the \clocc{} \dflcb remains tight.

\begin{figure}[!h]
    \centering
    \includegraphics[width=0.75\textwidth]{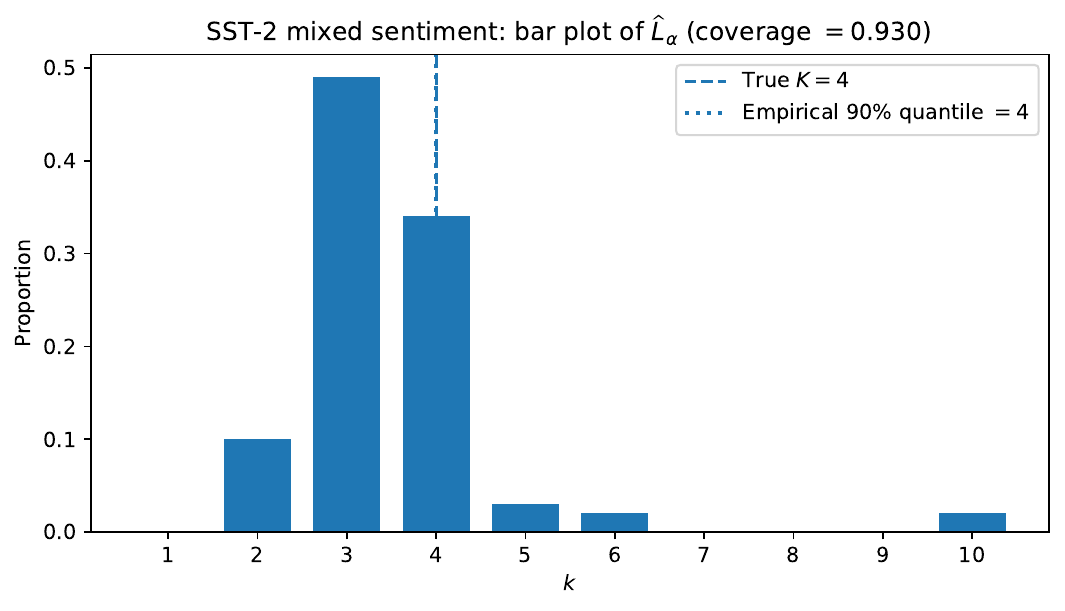}
    \caption{Bar plot of the \clocc{} \dflcb in the SST-2 mixed-sentiment experiment: despite challenging sentiment shifts, that are hard to detect, the \dflcb takes values $3$ and $4$ frequently, giving a tight lower bound.}
    \label{fig:sst2_clocc}
\end{figure}

\section{Discussion}\label{sec:discussion}

In this work, we have established that any valid \dfucb on $K$ is trivial, and have proposed the \clocc{} algorithm to construct a valid \dflcb on the same. The \clocc{} algorithm offers a flexible framework to the analyst: by choosing the CPP score efficiently, one can obtain tight lower bounds. We have also provided several practical variants of the \clocc{} framework that are both computationally and statistically efficient.

Since the \clocc{} framework proceeds by defining $p$-values $p_\bt$ for each candidate $\bt$, we also obtain, as a by-product, a valid confidence set for the changepoint locations, namely $\bar{C}_{1-\alpha}^{\conch-\text{multi}}$. This may be of independent interest in several applications.

An important future direction would be to allow temporal dependence in the data sequence and construct valid lower bounds for real-world dependent data sequences, which will help enhance the practicality of our approach.

\subsection*{Acknowledgment}
An AI model was used during the preparation of this manuscript to assist with language editing, presentation, and identifying potential mathematical inconsistencies. The authors have carefully reviewed the manuscript and take full responsibility for its technical content and conclusions.

\bibliographystyle{chicago}
\bibliography{bibliography}

\newpage
\appendix

\IncMargin{1.2em}
\begin{algorithm}[!h]
    \caption{\clocc-exact: \clocc{} with randomized exact $p$-values}
    \label{alg:clocc_exact}
    \KwIn{$(X_t)_{t=1}^n$ (data); $1-\alpha$ (target coverage);
    $S:\mathcal X^n\times\bigcup_{k=0}^{n-1}\mathcal T_k\to\overline{\R}$ (CPP score)}
    \KwOut{$\hat L_{\alpha}^{\clocc\text{-exact}}$}

    $\hat L_{\alpha}^{\clocc\text{-exact}}\gets n-1$\;

    \For{$k\in\{0,\ldots,n-1\}$}{
        $\bar p_{(k)}\gets 0$\;

        \ForEach{$\bt\in\mathcal T_k$}{
            Construct the split-permutation group $\Pi_{\bt}$\;
            Sample $U_{\bt}\sim\mathrm{Unif}[0,1]$ independently of the data\;

            $\bar p_{\bt}
            \gets
            \frac{1}{|\Pi_{\bt}|}
            \left[
            \sum_{\pi\in\Pi_{\bt}}
            \One{S(\pi(\bX),\bt)<S(\bX,\bt)}
            +
            U_{\bt}
            \sum_{\pi\in\Pi_{\bt}}
            \One{S(\pi(\bX),\bt)=S(\bX,\bt)}
            \right]$\;

            $\bar p_{(k)}
            \gets
            \max\{\bar p_{(k)},\bar p_{\bt}\}$\;
        }

        \If{$\bar p_{(k)}>\alpha$}{
            $\hat L_{\alpha}^{\clocc\text{-exact}}\gets k$\;
            \Return{$\hat L_{\alpha}^{\clocc\text{-exact}}$}\;
        }
    }

    \Return{$\hat L_{\alpha}^{\clocc\text{-exact}}$}\;
\end{algorithm}
\DecMargin{1.2em}

\IncMargin{1.2em}
\begin{algorithm}[!h]
    \caption{\clocc-MC: \clocc{} with random permutations}
    \label{alg:clocc_MC}
    \KwIn{$(X_t)_{t=1}^n$ (data); $1-\alpha$ (target coverage);
    $M$ (number of permutations);
    $S:\mathcal X^n\times\bigcup_{k=0}^{n-1}\mathcal T_k\to\overline{\R}$ (CPP score)}
    \KwOut{$\hat L_{\alpha}^{\clocc\text{-MC}}$}

    \For{$k\in\{0,\ldots,n-1\}$}{
        $\hat p_{(k)}\gets 0$\;

        \ForEach{$\bt\in\mathcal T_k$}{
            Construct the split-permutation group $\Pi_{\bt}$\;

            \For{$m\in[M]$}{
                Sample
                $\pi^{(m)}\sim\mathrm{Unif}(\Pi_{\bt})$\;
                Evaluate
                $S(\pi^{(m)}(\bX),\bt)$\;
            }

            $\hat p_{\bt}
            \gets
            \frac{
            1+
            \sum_{m=1}^M
            \One{
                S(\pi^{(m)}(\bX),\bt)
                \le
                S(\bX,\bt)
            }
            }{M+1}$\;

            $\hat p_{(k)}
            \gets
            \max\{\hat p_{(k)},\hat p_{\bt}\}$\;
        }

        \If{$\hat p_{(k)}>\alpha$}{
            $\hat L_{\alpha}^{\clocc\text{-MC}}\gets k$\;
            \Return{$\hat L_{\alpha}^{\clocc\text{-MC}}$}\;
        }
    }
\end{algorithm}
\DecMargin{1.2em}

\IncMargin{1.2em}
\begin{algorithm}[!h]
    \caption{\clocc-split: sample-split \clocc{}}
    \label{alg:clocc_split}
    \KwIn{$(X_t)_{t=1}^n$ (data); $1-\alpha$ (target coverage);
    score-learning algorithm $\mathcal A$}
    \KwOut{$\hat L_{\alpha}^{\clocc\text{-split}}$}

    $\mathcal I_1
    \gets
    \{i\in[n]:i\text{ is odd}\}$,\,
    $\mathcal I_2
    \gets
    \{i\in[n]:i\text{ is even}\}$\;

    $\mathcal D_1
    \gets
    (X_i)_{i\in\mathcal I_1}$\;

    $\bY
    \gets
    (X_i)_{i\in\mathcal I_2}$,
    ordered according to the original timeline\;

    Learn
    $\hat S\gets\mathcal A(\mathcal D_1)$
    using only the training split $\mathcal D_1$\;

    Hold $\hat S$ fixed for the remainder of the algorithm\;

    Let $m\gets|\mathcal I_2|$\;

    \For{$k\in\{0,\ldots,m-1\}$}{
        $\hat p_{(k)}^{\mathrm{split}}\gets0$\;

        \ForEach{$\bt\in\mathcal T_k^{(m)}$}{
            Construct the split-permutation group
            $\Pi_{\bt}^{(m)}$\;

            $p_{\bt}^{\mathrm{split}}
            \gets
            \frac{1}{|\Pi_{\bt}^{(m)}|}
            \sum_{\pi\in\Pi_{\bt}^{(m)}}
            \One{
                \hat S(\pi(\bY),\bt)
                \le
                \hat S(\bY,\bt)
            }$\;

            $\hat p_{(k)}^{\mathrm{split}}
            \gets
            \max\{
                \hat p_{(k)}^{\mathrm{split}},
                p_{\bt}^{\mathrm{split}}
            \}$\;
        }

        \If{$\hat p_{(k)}^{\mathrm{split}}>\alpha$}{
            $\hat L_{\alpha}^{\clocc\text{-split}}\gets k$\;
            \Return{$\hat L_{\alpha}^{\clocc\text{-split}}$}\;
        }
    }
\end{algorithm}
\DecMargin{1.2em}

\IncMargin{1.2em}
\begin{algorithm}[t]
    \caption{Practical \clocc-split implementation using classification}
    \label{alg:clocc_split_classification}
    \KwIn{$(X_t)_{t=1}^n$ (data); $1-\alpha$ (target coverage);
    changepoint localization procedure $\mathcal A$;
    multiclass classification procedure $\mathcal C$}
    \KwOut{$\hat L_{\alpha}^{\clocc\text{-split}}$}

    Form the interlaced split $\mathcal D_1\gets (X_1,X_3,\ldots)$ and $\mathcal D_2\gets (X_2,X_4,\ldots)$
    
    Apply $\mathcal A$ to $\mathcal D_1$ to obtain a reference segmentation $\hat{\bxi}
    =(\hat\xi_1,\ldots,\hat\xi_{\hat K})$

    Assign each observation in $\mathcal D_1$ the segment label induced by
    $\hat{\bxi}\;$

    Train the multiclass classifier $\mathcal C$ on $\mathcal D_1$ to predict
    these estimated segment labels

    From the fitted classifier, construct functions
    $\hat r_1,\ldots,\hat r_{\hat K+1}$ satisfying
    \[
    \hat r_1(x)\equiv0,
    \qquad
    \hat r_j(x)\approx
    \log\frac{f_j(x)}{f_1(x)},
    \quad j=2,\ldots,\hat K+1,
    \]
    using the classifier logits or log-probabilities

    Map the reference segmentation $\hat{\bxi}$ from $\mathcal D_1$
    to the corresponding locations
    $0=\hat\xi^{(2)}_0<\hat\xi^{(2)}_1<\cdots<
    \hat\xi^{(2)}_{\hat K}<\hat\xi^{(2)}_{\hat K+1}=|\mathcal D_2|$
    on the calibration split

    Define the frozen classifier-based CPP score on
    $\mathbf{y}=(y_1,\ldots,y_{|\mathcal D_2|})$ by
    \[
    \hat S(\mathbf{y})
    \gets \sum_{j=1}^{\hat K+1} -\left(
    \sum_{i=\hat\xi^{(2)}_{j-1}+1}^{\hat\xi^{(2)}_j}
    \hat r_j(y_i)\right)\;
    \]

    Hold $\hat S$ fixed and run \clocc{} on $\mathcal D_2$ using
    $\hat S$ as the CPP score

    \Return{$\hat L_{\alpha}^{\clocc\text{-split}}$}
\end{algorithm}
\DecMargin{1.2em}

\IncMargin{1.2em}
\begin{algorithm}[t]
    \caption{\clocc-SEG: \clocc{} with segmentwise test}
    \label{alg:clocc_seg}
    \KwIn{$(X_t)_{t=1}^n$ (data); $1-\alpha$ (target coverage);
    intervalwise score functions
    $\{A_{s,t}:1\le s<t\le n\}$}
    \KwOut{$\hat L_{\alpha}^{\clocc\text{-SEG}}$}

    \For{$1\le s<t\le n$}{
        Let $\mathcal S_{[s,t]}$ denote the set of all permutations
        of the indices in $[s,t]$\;
        Compute the intervalwise $p$-value
        \[
        p_{s,t}\gets
        \frac{1}{|\mathcal S_{[s,t]}|}
        \sum_{\pi\in\mathcal S_{[s,t]}}
        \One{
        A_{s,t}(X_{\pi(s)},\ldots,X_{\pi(t)})
        \le
        A_{s,t}(X_s,\ldots,X_t)
        }\;.
        \]
    }
    Set $p_{s,s}\gets1$ for all $s\in[n]$\;

    \For{$k\in\{0,\ldots,n-1\}$}{
        $p_{(k)}^{\mathrm{seg}}\gets0$\;

        \ForEach{$\bt=(t_1,\ldots,t_k)\in\mathcal T_k$}{
            Set $t_0\gets0$ and $t_{k+1}\gets n$\;

            $p_{\bt}^{\mathrm{seg}}
            \gets
            \overline F_{\chi^2_{2(k+1)}}
            \left(
            -2
            \sum_{j=1}^{k+1}
            \log
            p_{t_{j-1}+1,t_j}
            \right)$\;

            $p_{(k)}^{\mathrm{seg}}
            \gets
            \max\{
                p_{(k)}^{\mathrm{seg}},
                p_{\bt}^{\mathrm{seg}}
            \}$\;
        }

        \If{$p_{(k)}^{\mathrm{seg}}>\alpha$}{
            $\hat L_{\alpha}^{\clocc\text{-SEG}}\gets k$\;
            \Return{$\hat L_{\alpha}^{\clocc\text{-SEG}}$}\;
        }
    }
\end{algorithm}
\DecMargin{1.2em}

\section{Proofs}
\subsection{Proof of Theorem~\ref{thm:ucb_impossibility}}
Fix $x\in \Xcal^n$ and $\varepsilon>0$. Then, we can define distributions $P_{1,\varepsilon},\ldots,P_{n,\varepsilon}$
on the space $\mathcal{X}$ such that $P_{i,\varepsilon}\neq P_{i+1,\varepsilon}$ for all $i\in [n-1]$ and that
\[
\P_{X\sim P_{i,\varepsilon}}(X=x_i)\ge 1-\varepsilon/n, \qquad \text{for all~} i\in [n].
\]

These distributions can be constructed as follows. For each $i\in [n-1]$, choose some $x'_i\in \Xcal$ such that $x'_i\neq x_i$. This exists since $\Xcal$ contains at least two measurable points. Now, let $\eta_1\neq \ldots\neq \eta_n$ be distinct numbers in $(0,\min\{\varepsilon/n,1/3\})$ and define
\[
P_{i,\varepsilon}:=(1-\eta_i)\delta_{x_i}+\eta_i\delta_{x_i'}.
\]
Here $\delta_x$ denotes the point mass at $x\in\Xcal$. Then we have 
$\P_{X\sim P_{i,\varepsilon}}(X=x_i)\ge 1-\varepsilon/n$ for every $i\in [n]$. Moreover, if $x_i=x_{i+1}$, $P_{i,\varepsilon}\neq P_{i+1,\varepsilon}$ since $\eta_i\neq\eta_{i+1}$. On the other hand, if $x_i\neq x_{i+1}$, they are distinct since each distribution places more than $2/3$ of its mass on a different point. Thus, $P_{i,\varepsilon}\neq P_{i+1,\varepsilon}$ for every $i\in[n-1]$, as required.

By construction, $P_{\varepsilon}=\prod_{i=1}^n P_{i,\varepsilon}\in \cup_{k=1}^{n-1}\mathfrak{P}_{k}$ satisfies Assumption~\ref{ass:exchangeability} with $n-1$ changepoints. Therefore, by the theorem hypothesis, we get that 
\[
\P_{\bX\sim P_\varepsilon}(\hat{U}_\alpha(\bX,\zeta)\ge n-1)\ge 1-\alpha.
\]
Further, under $P_\varepsilon$, $\bX$ takes the value $\bx$ with probability at least $(1-\varepsilon/n)^n\ge 1-\varepsilon$ and recall that $\hat{U}_\alpha\in \{1,\ldots,n-1\}$. Consequently, 
\[
\P(\hat{U}_\alpha(\bx,\zeta)= n-1)\ge \frac{1-\alpha-\varepsilon}{1-\varepsilon}.
\]
Since this is true for any $\varepsilon>0$, taking $\varepsilon\to 0$ proves the result.

$\hfill\square$
\subsection{Proof of results from Section~\ref{sec:method}}

\subsubsection{Proof of Lemma~\ref{lem:validity_of_conch_multi}}

The proof follows analogous to the proof of Theorem~3.1 in \cite{hore2026conformal}. We include it here for completeness.

Fix $\bt\in \cup_{k=1}^{n-1}\mathcal{T}_k$. By definition, under the null $\tilde{\Hcal}_{0,\bt}$, $\pi(\bX)\overset{d}{=}\bX$ for any $\pi\in \Pi_{\bt}$. We start by defining a function $p_{\bt}:\Xcal^n\to [0,1]$ by
\[
p_{\bt}(\bx):=\frac{1}{|\Pi_{\bt}|}\sum_{\pi\in \Pi_{\bt}}\One{S(\pi(\bx),\bt)\leq S(\bx,\bt)},
\]
and note that $p_{\bt}\equiv p_{\bt}(\bX)$. Therefore, it follows that
\begin{align*}
    \P_{\bt}\left(p_{\bt}(\bX)\leq \alpha\right)&=\frac{1}{|\Pi_{\bt}|}\sum_{\pi\in \Pi_{\bt}} \P_{\bt}\left(p_{\bt}(\pi(\bX))\leq \alpha\right)\\
    &=\E_{\bt}\left[\frac{1}{|\Pi_{\bt}|}\sum_{\pi\in \Pi_{\bt}}\One{p_{\bt}(\pi(\bX))\leq \alpha}\right]\\
    &=\E_{\bt}\left[\frac{1}{|\Pi_{\bt}|}\sum_{\pi\in \Pi_{\bt}}\One{\frac{1}{|\Pi_{\bt}|}\sum_{\pi^\prime\in \Pi_{\bt}} \One{S(\pi^\prime(\bX), {\bt})\leq S(\pi(\bX),{\bt})}\leq \alpha}\right]\leq \alpha,
\end{align*}
where the penultimate step follows by noting that $\pi\circ \Pi_{\bt}=\Pi_{\bt}$, and the last step is a deterministic inequality. 
$\hfill\square$

\subsubsection{Proof of Theorem~\ref{thm:validity_clocc}: validity of \clocc}
In Lemma~\ref{lem:validity_of_conch_multi}, we have proved that under the null $\tilde{H}_{0,\bt}$, $p_\bt$ is a valid $p$-value. Further, we recall that $\Hcal_{0,k}=\bigcup_{\bt\in I_k}\tilde{\Hcal}_{0,\bt}$, and that $p_{(k)}=\max_{\bt\in I_k} p_{\bt}$.

Therefore, it immediately follows that $\P_k(p_{(k)}\le \alpha)\le \alpha$.
This completes the proof. 

$\hfill\square$
\subsubsection{Proof of Theorem~\ref{thm:universality_clocc}: universality result}\label{app:proof of universality theorem}

Fix $n\in \N$ and suppose we observe $\bX=(X_1,\ldots,X_n)$. First, given a valid \dflcb $L$, we consider the score
\[
S(\bx,\bt) = \One{\mathrm{len}(t)\ge  L(\bx)} \in \{0,1\},
\]
for any $t \in \cup_{k=0}^{n-1}\mathcal{T}_k$, where for any vector $\bt$, we write $\mathrm{len}(\bt)$ to denote the length of the vector $\bt$. Here, we will interpret $\mathrm{len}(\emptyset)$ as $0$. Let $\hat{L}_\alpha^{
\clocc}(\bX)$ be the \dflcb returned by the \clocc{} algorithm run with the data $\bX$ and the score $S(\bx,\bt)$.
We will show that the $\hat{L}_\alpha^{
\clocc}(\bX)=L(\bX)$.

We start with showing that $\hat{L}_\alpha^{
\clocc}(\bX) \le L(\bX)$, that is, if $k = L(\bX)$, then we show that $p_{(k)} > \alpha$, where $p_{(k)}$ is as defined in \eqref{eq:pvalue_clocc}.
This is immediate by observing that if $k = L(\bX)$, then $S(\bX,\bt) = 1$ for every $\bt\in \cup_{k=1}^{n-1}\mathcal{T}_k$, and consequently, 
\[
p_{\bt} = \frac{1}{|\Pit|}\sum_{\pi \in \Pit} \One{S(\pi(\bX),\bt) \leq S(\bX,\bt)}
     = \frac{1}{|\Pit|}\sum_{\pi \in \Pit} \One{S(\pi(\bX),\bt) \leq 1} = 1.
\]
Consequently, $p_{(k)}=1$, which proves this part.

Next, we show the other side that $\hat{L}_\alpha^{
\clocc}(\bX) \ge L(\bX)$, i.e., for every $k<L(\bX)$, $p_{(k)} \leq \alpha$. If $L(\bX)=0$, the statement is vacuously true and the result follows. Therefore, in the following part, we assume that $L(X)>0$.
Fix any $k<L(\bX)$.
We first claim that for any tuple $t\in \mathcal{T}_k$ and any vector $\bx \in \mathcal{X}^n$,
\begin{equation}\label{eq:claim}
    \frac{1}{|\Pit|}\sum_{\pi \in \Pit} \One{k \ge L(\pi(\bx))} \geq 1 - \alpha.
\end{equation}

To prove this claim, fix any such tuple $t\in \cup_{k=0}^{n-1}\mathcal{T}_k$, and sample $\pi$ uniformly from the set of permutations $\Pit$. Define $\tilde{\bX} := (\tilde{X}_1, \ldots, \tilde{X}_n) := \pi(\bx)$. Conditional on the multisets $\{x_{t_{j-1}+1}, \ldots, x_{t_{j}}\}$ for $j=1,\ldots, k+1$, we have that 
$\tilde{\bX}$ is exchangeable within each of the $k$ segments. Merging some of these segments can retain exchangeability, and therefore $\tilde\bX$ must have at most $k$ changepoints.

Moreover, conditional on these multisets, the segments are also independent. Consequently,
\[
\P_{\pi \sim \text{Unif}(\Pit)}\big(k\ge L(\tilde{\bX}) \,\big|\, \mathrm{multisets}\big) \geq 1 - \alpha,
\]
or equivalently, \eqref{eq:claim} holds.

Returning to the main proof, observe that if $k< L(\bX)$, then $S(\bX,\bt) = 0$ for any such $\bt\in \mathcal{T}_k$. Consequently, for any such $\bt$,
\begin{align*}
p_{\bt} &= \frac{1}{|\Pit|}\sum_{\pi \in \Pit} \One{S(\pi(\bX),\bt) \leq S(\bX,\bt)} \\
    &= \frac{1}{|\Pit|}\sum_{\pi \in \Pit} \One{S(\pi(\bX),\bt) \leq 0}
     = \frac{1}{|\Pit|}\sum_{\pi \in \Pit} \One{k < L(\pi(\bX))} <\alpha,
\end{align*}
where the last step follows from \eqref{eq:claim}. Since this holds for any tuple $\bt\in \mathcal{T}_k$, we have that 
\[
p_{(k)}=\max_{\bt \in \mathcal{T}_k} p_{\bt}< \alpha.
\]
Since $k$ was chosen arbitrarily, 
this completes the proof. \hfill$\square$

\subsection{Proof of results from Section~\ref{sec:practical_clocc}}

\subsubsection{Proof of Lemma~\ref{lem:validity_of_conch_multi_exact}}
Let $F$ denote the distribution of $S(\pi(\bX),\bt)$ with $\pi \sim \textnormal{Unif}(\Pi_{\bt})$, conditional on the multisets 
\[
M_{j}=\{X_{t_{j-1}+1},\ldots,X_{t_j}\},\qquad j=1,\ldots, K-1.
\]
Equivalently, we can write that
\[
\bar{p}_t = \lim_{y \uparrow S(\bX,]bt)} F(y) + U \bigl(F(S(\bX,\bt)) - \lim_{y \uparrow S(\bX,\bt)} F(y)\bigr).
\]
Under $\tilde{\Hcal}_{0,t}$, we have $S(\bX,\bt) \overset{d}{=} S(\pi(\bX),\bt)$ conditional on $M_1,\ldots,M_{K+1}$. Hence, by \citet[Lemma~E.1]{dandapanthula2025conformal}, the $p$-value $\bar{p}_t$, conditional on $M_1,\ldots,M_{K+1}$, follows $\textnormal{Unif}[0,1]$ (see also \citealp{brockwell2007universal}). 
Therefore,
\[
\P_{\bt}(\bar{p}_{\bt} \le \alpha) 
= \E_{\bt}\!\left[\P_{\bt}\!\left(\bar{p}_{\bt} \le \alpha \mid M_1,\ldots,M_{K+1}\right)\right]
= \E_{\bt}[\alpha] = \alpha.
\]
This completes the proof.$\hfill\square$
\subsubsection{Proof of Theorem~\ref{thm:coverage-clocc_MC}}

Given permutations $\pi_{1,\bt},\ldots,\pi_{M,\bt}\in \Pi_{\bt}$, we start by defining the function
\[
\hat{p}_{\bt}(\bx;\pi_{1,\bt},\ldots,\pi_{M,\bt}):=\frac{1+\sum_{k=1}^M\One{S(\pi_{k,t}(\bx),\bt)\leq S(\bx,\bt)}}{1+M},
\]
Now, consider an independent draw $\pi_{0,\bt}\sim \textnormal{Unif}(\Pi_\bt)$ and note that with $\pi_{1,\bt},\ldots,\pi_{M,\bt}\overset{iid}{\sim}\textnormal{Unif}(\Pi_{\bt})$, we have that $(\pi_{1,\bt},\ldots,\pi_{M,\bt})\overset{d}{=}(\pi_{0,\bt}\circ\pi_{1,\bt},\ldots,\pi_{0,\bt}\circ\pi_{M,\bt})$.
Moreover, conditional on $\pi_{0,\bt},\pi_{1,\bt},\ldots,\pi_{M,\bt}$, $\bX\overset{d}{=}\pi_{0,\bt}(\bX)$ under the null $\tilde{\mathcal{H}}_{0,\bt}$. Consequently,
\begin{multline*}
    \hat{p}_{\bt}(\bX;\pi_{1,\bt},\ldots,\pi_{M,\bt})\overset{d}{=}\hat{p}_{\bt}(\bX;\pi_{0,\bt}\circ\pi_{1,\bt},\ldots,\pi_{0,\bt}\circ\pi_{M,\bt})\overset{d}{=}\hat{p}_{\bt}(\pi_{0,\bt}(\bX);\pi_{0,\bt}\circ\pi_{1,\bt},\ldots,\pi_{0,\bt}\circ\pi_{M,\bt}).
\end{multline*}
Finally, note that for $\hat{p}_{\bt}$, defined in \eqref{eq:mc_pvalue_conch_multi},  $\hat{p}_{\bt}\equiv \hat{p}_{\bt}(\bX;\pi_{1,\bt},\ldots,\pi_{M,\bt})$, and therefore,
\begin{align*}
 \hat{p}_{\bt}(\bX;\pi_{1,\bt},\ldots,\pi_{M,\bt})&\overset{d}{=}\hat{p}_{\bt}(\pi_{0,\bt}(\bX);\pi_{0,\bt}\circ\pi_{1,\bt},\ldots,\pi_{0,\bt}\circ\pi_{M,\bt})\\
    &=\frac{1+\sum_{k=1}^M\One{S(\pi_{k,\bt}(\bX),\bt)\leq S(\pi_{0,\bt}(\bX),\bt)}}{M+1}\\
    &=\frac{\sum_{k=0}^M\One{S(\pi_{k,\bt}(\bX),\bt)\leq S(\pi_{0,\bt}(\bX),\bt)}}{M+1},
    \end{align*}
i.e., the rank of $S(\pi_{0,\bt}(\bX),\bt)$ in the exchangeable collection \[\{S(\pi_{0,\bt}(\bX),\bt),S(\pi_{1,\bt}(\bX),\bt),\ldots, S(\pi_{M,\bt}(\bX),\bt)\}.\] Consequently, this gives us
\[
\P_{\bt}\left(\hat{p}_{\bt}=\hat{p}_{\bt}(\bX;\pi_{1,\bt},\ldots,\pi_{M,\bt})\leq \alpha\right)\leq \alpha.
\]
This proves the result. 
$\hfill\square $
    
\subsection{Proof of results from Section~\ref{sec:CPP_score}}

\subsubsection{Proof of Proposition~\ref{prop:score-properties}}
Fix $k\in \{1,\ldots,n-1\}$ and $t\in \mathcal{T}_k$. If $S$ satisfies $S(\cdot,\bt)=S(\pi(\cdot),\bt)$ for all $\pi\in \Pi_t$, then by \eqref{eq:pvalue_conch_multi} note that $p_\bt$ is identically equal to $1$. This proves the first part.

For the second part, fix $t\in \cup_{k=1}^{n-1}\mathcal{T}_k$. By \eqref{eq:pvalue_conch_multi},
\begin{align*}
    p_{\bt,1}&=\frac{1}{|\Pi_\bt|}\sum_{\pi\in \Pi_\bt} \One{S(\pi(\bX),\bt)\leq S(\bX,\bt)}\\
    p_{\bt,2}&=\frac{1}{|\Pi_\bt|}\sum_{\pi\in \Pi_\bt} \One{f(S(\pi(\bX),\bt))\leq f(S(\bX,\bt))}.
\end{align*}
Since by the hypothesis, $f$ is non-decreasing, $S(\pi(\bX),\bt)\leq S(\bX,\bt)$ implies $f(S(\pi(\bX),\bt))\leq f(S(\bX,\bt))$,
and therefore $p_{\bt,1}\leq p_{\bt,2}$. This holds for all $\bt\in \cup_{k=1}^{n-1}\mathcal{T}_k$, it further follows that $p_{(k),1}\le p_{(k),2}$ where we let for $k\in \{1,\ldots,n-1\}$
\[
p_{(k),1}=\max_{\bt\in \mathcal{T}_k} p_{\bt,1},\qquad p_{(k),2}=\max_{\bt\in \mathcal{T}_k} p_{\bt,2}
\]  
Consequently, it follows that $L_1\ge L_2$.
$\hfill\mathsf{\square}$

\subsubsection{Proof of Theorem~\ref{thm:optimal_score}}
We start by noting that for any strictly increasing function $f:\R\to\R$, by Proposition~\ref{prop:score-properties}~(ii),
\[
\bar{\Ccal}_{1-\alpha}^{\conch\text{-multi}}(f(S^{\mathrm{OPT}}))=\bar{\Ccal}_{1-\alpha}^{\conch\text{-multi}}(S^{\mathrm{OPT}}).
\]
Therefore, for the rest of the proof, we take $f$ as the identity function. Now, note that
\begin{align}\label{eq:length_as_a_sum}
    \E_{\tilde{\Hcal}_{0,\bxi}\,\cap\,\Pcal_{\mathrm{IID}}}\!\left[|\bar{\Ccal}_{1-\alpha}^{\conch\text{-multi}}(S^{\mathrm{OPT}})|\right]=\E_{\tilde{\Hcal}_{0,\bxi}\,\cap\,\Pcal_{\mathrm{IID}}}\!\left[\one{\bar{p}_{\bxi}> \alpha}\right]+\sum_{\bt\in \cup_{r=1}^{n-1}\mathcal{T}_r,\bt\neq \bxi}\E_{\tilde{\Hcal}_{0,\bxi}\,\cap\,\Pcal_{\mathrm{IID}}}\!\left[\one{\bar{p}_{\bt}> \alpha}\right].
\end{align}
Let $\phi_\bt(\bX;S):=\one{\bar{p}_{\bt}\le \alpha}$, where $\bar{p}_{\bt}$ be the conformal $p$-value~\eqref{eq:pvalue_conch_multi_exact} computed on the data sequence $\bX$ with CPP score $S$.

Now, fix such a candidate changepoint location $\bt\neq \bxi$.
Let us define the corresponding multisets
\[
M_{j}=\{X_{t_{j-1}+1},\ldots,X_{t_j}), \qquad j\in [K+1],
\]
Therefore, by~\eqref{eq:length_as_a_sum}, in order to minimize the length $|\bar{\Ccal}_{1-\alpha}^{\conch\text{-multi}}|$, it suffices to maximize
\[
\E_{\tilde{\Hcal}_{0,\bxi}\,\cap\,\Pcal_{\mathrm{IID}}}\!\left[\phi_\bt(\bX;S)\ \middle| \ M_1,\ldots,M_{k+1}\right]\]
for each $\bt\neq \bxi$ over all choices of CPP score $S$.
From Lemma~\ref{lem:validity_of_conch_multi_exact}, we know that $\bar{p}_{\bt}$ is exactly a uniform random variable, conditional on the multisets $M_1,\ldots,M_{k+1}$. Write $P^{[\bv]}$ to denote the distribution of $\bX$ conditional on multisets $M_1,\ldots,M_{k+1}$ under the null $\tilde\Hcal_{\bv}$ for any $\bv\in \cup_{r=1}^{n-1}\mathcal{T}_r$, and let $\Hcal'_{\bv}$ to hypothesize that 
\[
\bX\mid M_1, \ldots, M_{k+1}\sim P^{[\bv]}.
\]
With this notation, maximizing $\E_{\tilde{\Hcal}_{0,\bxi}\,\cap\,\Pcal_{\mathrm{IID}}}\!\left[\phi_\bt(\bX;S)\ \middle| \ M_1,\ldots,M_{k+1}\right]$ is equivalent to maximizing $\E_{\Hcal'_{\bxi}}\!\left[\phi_\bt(\bX;S)\right]$, or defining the optimal conformal score for testing $\Hcal'_{\bt}$ against $\Hcal'_{\bxi}$ under the aforementioned framework. Applying the following lemma then completes the proof.
\hfill$\square$

\begin{lemma}\label{lem:conformal_np_lemma_multi}
Fix $\bt=(t_{1},\ldots,t_{k}),\bv=(v_1,\ldots,v_{\ell})\in \cup_{r=1}^{n-1}\mathcal{T}_r$ with $\bt\neq \bv$. The power $\E_{\tilde\Hcal_{\bv}}[\phi_{\bt}(\bX;s)]$ is maximized by the score function
\[
s^\star(\bx) := \frac{
\prod_{j=1}^{k+1}\prod_{i=t_{j-1}+1}^{t_j}f_j(x_i)
}{
\prod_{j=1}^{\ell+1}\prod_{i=v_{j-1}+1}^{v_j}f_j(x_i)
}
\]
\end{lemma}
\begin{proof}
In the above setup, we consider the following hypothesis testing problem:
\[
\Hcal^\prime_{0}:\bX\mid(M_1,\ldots,M_{k+1})\sim P^{[\bt]}
\quad \text{vs.} \quad
\Hcal^\prime_{1}:\bX\mid(M_1,\ldots,M_{k+1})\sim P^{[\bv]}.
\]

Given samples $\bX\in\Xcal^n$, observe that
\[
\frac{\mathsf{d}\bigl(P^{[\bv]}\bigr)}{\mathsf{d}\bigl(P^{[\bt]}\bigr)}(\bX)\propto 
\frac{\prod_{j=1}^{\ell+1}\prod_{i=v_{j-1}+1}^{v_j}f_j(x_i)}{\prod_{j=1}^{k+1}\prod_{i=t_{j-1}+1}^{t_j}f_j(x_i)}
= {s^\star(\bX)}^{-1}.
\]

By the Neyman--Pearson lemma \citep[Theorem~3.2.1~(ii)]{lehmann2005testing}, any test $\phi(\bX)$ that attains exact level $\alpha$ under $\Hcal_{0}^\prime$ and satisfies
\begin{equation}\label{eqn:NP_optimal_form}
    \phi(\bX)=
\begin{cases}
    1 & \text{if } {s^\star(\bX)}^{-1} > \tau_{\alpha},\\[4pt]
    0 & \text{if } {s^\star(\bX)}^{-1} < \tau_{\alpha},
\end{cases}
\end{equation}
for some threshold $\tau_{\alpha}\in \R$, is most powerful for testing $\Hcal_{0}^\prime$ against $\Hcal_{1}^\prime$.

We know that the test $\phi_\bt(\cdot;s)=\One{p_\bt(s)\leq \alpha}$ controls Type~I error exactly at level $\alpha$ under $\Hcal_{0}^\prime$ for any score function $s$. Therefore, to establish optimality of $s^\star$, it suffices to show that $\phi_\bt(\cdot;s^\star)$ has the form given in \eqref{eqn:NP_optimal_form}.

Let $\bX_{\pi}=\pi(\bX)$ for $\pi\sim \mathrm{Unif}(\Pi_\bt)$, and let $F_{s^\star(\bX_\pi)}$ denote the conditional cumulative distribution function of $s^\star(\bX_{\pi})$ given $\bX$. Define
\[
\tau_{\alpha}:=\inf\{y\in \R: F_{s^\star(\bX_\pi)}(y)\geq \alpha\}.
\]
By the definition of $\bar{p}_\bt$ in \eqref{eq:pvalue_conch_multi_exact}, we have
\begin{align*}
    {s^\star(\bX)}^{-1} < \tau_{\alpha} &\implies \bar{p}_\bt(s^\star)\leq \alpha,\\
    {s^\star(\bX)}^{-1} \ge \tau_{\alpha} &\implies \bar{p}_\bt(s^\star)> \alpha,
\end{align*}
which establishes the desired form. This completes the proof.
\end{proof}

\end{document}